\documentclass[letterpaper,11pt]{article}
\usepackage{amssymb,amsbsy,amsmath,amscd,amsthm,amsfonts}
\usepackage{cite}
\usepackage{enumerate}
\usepackage{dsfont}
\usepackage{graphicx}
\usepackage{longtable}
\usepackage{tablefootnote}
\usepackage{standalone}
\usepackage{fancyhdr}
\usepackage{float}
\usepackage{caption}
\usepackage{subcaption}
\usepackage{url}
\usepackage[margin=1.35in]{geometry}

\RequirePackage[colorlinks = true,
            linkcolor = black,
            urlcolor  = black,
            citecolor = black,
            anchorcolor = black]{hyperref}

\newcommand{\Expect}{\mathbb{E}}
\newcommand{\expect}[1]{\Expect\left[#1\right]}
\newcommand{\Prob}{\mathbb{P}}
\newcommand{\prob}[1]{\Prob\left[#1\right]}

\newcommand{\Bern}{\mathrm{Bern}}
\newcommand{\Bin}{\mathrm{Bin}}

\newcommand{\indc}{\mathbf{1}}

\newcommand{\calA}{{\mathcal{A}}}

\newcommand{\calC}{{\mathcal{C}}}
\newcommand{\calD}{{\mathcal{D}}}

\newcommand{\calF}{{\mathcal{F}}}
\newcommand{\calG}{{\mathcal{G}}}

\newcommand{\calN}{{\mathcal{N}}}

\newcommand{\calR}{{\mathcal{R}}}
\newcommand{\calS}{{\mathcal{S}}}

\newcommand{\calU}{{\mathcal{U}}}
\newcommand{\calV}{{\mathcal{V}}}

\newtheorem{definition}{Definition}
\newtheorem{theorem}{Theorem}

\newtheorem{lemma}{Lemma}
\newtheorem{remark}{Remark}

\newtheorem{corollary}{Corollary}

\usepackage{prettyref}
\newrefformat{eq}{(\ref{#1})}
\newrefformat{thm}{Theorem~\ref{#1}}
\newrefformat{sec}{Section~\ref{#1}}
\newrefformat{algo}{Algorithm~\ref{#1}}
\newrefformat{fig}{Fig.~\ref{#1}}
\newrefformat{tab}{Table~\ref{#1}}
\newrefformat{rmk}{Remark~\ref{#1}}
\newrefformat{def}{Definition~\ref{#1}}
\newrefformat{cor}{Corollary~\ref{#1}}
\newrefformat{lmm}{Lemma~\ref{#1}}
\newrefformat{prop}{Proposition~\ref{#1}}
\newrefformat{app}{Appendix~\ref{#1}}
\newrefformat{ex}{Example~\ref{#1}}
\newrefformat{ass}{Assumption~\ref{#1}}

\providecommand{\keywords}[1]{\textbf{\text{Index terms ---}} #1}

\usepackage[]{algorithm2e}
\usepackage{authblk}
\title{Approximation and Estimation for High-Dimensional Deep Learning Networks\footnotetext{A portion of this paper was presented at the 2017 IEEE International Symposium on Information Theory in Aachen, Germany. See \cite{Klusowski2017} for the conference proceedings.}}
\author[1]{Andrew R. Barron\thanks{andrew.barron@yale.edu}}
\author[2]{Jason M. Klusowski\thanks{jason.klusowski@rutgers.edu}}
\affil[1]{Department of Statistics and Data Science, Yale University}
\affil[2]{Department of Statistics and Biostatistics, Rutgers University \textbf{--} New Brunswick}

\begin{document}


\maketitle

\begin{abstract}
It has been experimentally observed in recent years that multi-layer artificial neural networks have a surprising ability to generalize, even when trained with far more parameters than observations. Is there a theoretical basis for this? The best available bounds on their metric entropy and associated complexity measures are essentially linear in the number of parameters,
which is inadequate to explain this phenomenon. Here we examine the statistical risk (mean squared predictive error) of multi-layer networks with $\ell^1$-type controls on their parameters and with ramp activation functions (also called lower-rectified linear units).
In this setting, the risk is shown to be upper bounded by $[(L^3 \log d)/n]^{1/2}$, where $d$ is the input dimension to each layer, $L$ is the number of layers, and $n$ is the sample size. In this way, the input dimension can be much larger than the sample size and the estimator can still be accurate, provided the target function has such $\ell^1$ controls and that the sample size is at least moderately large compared to $L^3\log d$. The heart of the analysis is the development of a sampling strategy that demonstrates the accuracy of a sparse covering of deep ramp networks. Lower bounds show that the identified risk is close to being optimal.
\end{abstract}


\section{Introduction} \label{sec:introduction}

Good empirical performance of deep learning networks has been reported across various disciplines for difficult tasks in classification and prediction \cite{LeCun2015}. These successes have largely been buoyed by the ability of multi-layer networks to generalize well despite being able to fit any dataset, given enough parameters \textbf{---} an apparent contradiction to age-old statistical wisdom that warns against overfitting. This phenomenon is particularly striking when the input dimension is far greater than the available sample size, as is the case with many modern applications in molecular biology, medical imaging, and astrophysics, to name a few.
Despite a vast amount of effort that goes into training deep learning models, typically in an ad-hoc manner for anecdotal datasets, a unifying theory of their complex mechanisms has not yet caught up with these applied and practical developments.

As is generally true in statistical estimation, there is a trade-off between estimation error and descriptive model complexity relative to sample size. At the outset, one may be tempted to believe that the descriptive complexity of deep learning models is very large, in accordance with the large number of parameters that index each model. Fortunately, we will show that, although a generic deep network may be difficult to describe,
nevertheless, under suitable control on norms of the weights, it can be well approximated by a sparse representation, and this sparse representation comes from a subfamily that has a manageable cardinality. We will then use these small cardinality 
covers to balance the estimation error and complexity trade-off and thereby achieve (close to) optimal rates of estimation, in a minimax sense, in appropriate settings.

Prior results that seek to quantity different notions of model complexity typically produce unsavory statistical risk bounds for two main reasons reasons.

First, the functions classes that are approximated by deep networks are typically not suited for high-dimensional settings.
Indeed, minimax optimal rates for certain smooth function classes (e.g., Lipschitz, H\"older, Sobolev) degrade either with the number of inputs per layer, viz., $ O(n^{-\alpha_d}) $, where $ \alpha_d \rightarrow 0 $ as $ d $ approaches infinity, or in a similar way through the depth. 
Second, the complexity constants often scale exponentially with the depth or number of units per layer \cite{neyshabur2017, neyshabur2015, golowich2017, bartlett2017, arora2018}, which is problematic for high-dimensional or very deep networks.

Other works \cite{yarotsky2017, yarotsky2018, harvey2017} study the general approximation capabilities of deep networks using state-of-the-art VC dimension bounds $ cT L \log(T/L) \leq \mbox{VCdim}(T, L) \leq CT L \log T $ for depth $ L $ ramp networks, 
where $T$ is the number of weights. 
So the VC dimension is indeed linear in the number of parameters to within a log-factor. 
These results, again, when applied to a statistical learning setting, do not satisfactorily showcase the advantages of these model classes. 

Our perspective on function estimation is slightly different than, for example, \cite{schmidt2017}, who works with a function class known to be rich enough for flexible high-dimensional modeling, while, at the same time, ensuring its members also admit sparse representations by deep networks. Even still, associated minimax rates depend in the exponent on the level of sparsity and smoothness of the target function. Instead, we assume that the target function is equal to (or approximated well) by a deep network and ask, for such a family, what is the size of the smallest subfamily with members that can approximate an arbitrary network within a desired level of accuracy?

To the best of our knowledge, the risk bound here of order $ [(L^3\log d)/n]^{1/2}$ is the first to provide conditions for a fixed rate ($1/2$) in the exponent while having the numerator depend only logarithmically on the number of parameters and as a low order polynomial in the depth $L$.
These results give conditions such that desirably good performance of deep networks can be achieved under rather weak conditions on sample size, i.e., $ n $ large relative to $L^3\log d$. Indeed, the effect of $ L $ can be benign as many applications involve depths ranging from 2 or 3 to 22 \cite{szegedy2015} or, at the extreme end, 152 \cite{he2016}. On the other hand, $ d $ can be extremely large, possibly in the millions, but thankfully its effect is modulated by the presence of the logarithm. 

This paper is organized as follows. In \prettyref{sec:deep}, we formally introduce the notation and language of deep networks used throughout the article. In \prettyref{sec:variation}, we define new concepts of variation and average variation for multi-layer networks and subnetworks. In \prettyref{sec:cover} and \prettyref{sec:improve}, we outline our strategy for constructing and counting the number of sparse approximants of deep networks. New complexity constants that govern the quality of the approximation and corresponding metric entropy bounds are defined in \prettyref{sec:main}. The proof of our main result is furnished in \prettyref{sec:proof}, followed by a discussion on the multiplicity of representation of a given network function in \prettyref{sec:scaling} and how it affects our bounds. We compare our results with related literature in \prettyref{sec:discussion}. \prettyref{sec:examples} presents examples where our complexity constants can be explicitly calculated and shown to be independent of the depth $ L $ and number of inputs per layer $ d $. In \prettyref{sec:twolayer}, we specialize our general treatment of multi-layer networks to the case of two layers and provide a covering number for a large class of high-dimensional functions. In \prettyref{sec:statistics}, we use the metric entropy bounds from \prettyref{sec:main} to obtain rates for the minimax risk and generalization error of deep network function classes. In \prettyref{sec:optimal}, we show that these rates are close to being minimax optimal. Finally, some additional examples and proofs of supporting claims are given in \prettyref{app:appendix}.

\section{Deep Networks} \label{sec:deep}
Let $f(W,x)$ be the parameterized family of depth $L$ networks which map input vectors $x$ of dimension $d_{in}$ into output vectors of dimension $d_{out}$, where $f(W,x)$ either takes the form $W_1 \phi (W_2 \phi (\cdots W_{L-1} \phi (W_L x)))$ or the form $$\phi_{out}(W_1 \phi (W_2 \phi (\cdots W_{L-1} \phi (W_L x)))),$$ where  $\phi_{out}$ is any Lipschitz($1$) function, such as the fully-rectified linear function $\phi_{out}(z)
=\mbox{sgn}(z)\min\{|z|, 1\}$, which is applied at the output, and $\phi$ is the positive-part activation function (also known as the ramp function or lower-rectified linear unit or first order spline basis function with knot at $0$) applied at the internal layers. This positive-part function takes values $\phi(z)=z_+ = \max\{z,0\}$ for scalar inputs $z$. 
For vector inputs, our understanding is that $\phi$ is the vector-valued function that results from application of the positive part coordinate-wise (though below we will modify it so that half the coordinates of $\phi$ use the positive part and half use minus the positive part).

There are $d_{\ell}$ units on layer $\ell$ for $\ell= 0,1,2,\ldots,L$, with $d_0= d_{out}$ on the outermost layer, and $d_{L} = d_{in}$ input units on the innermost layer, where, for analysis convenience, 
we are letting $\ell$ specify the number of layers away from the output. It is typical practice to set $d_1,d_2,\ldots,d_{L-1}$ to be a common (possibly quite large) value $d$, at least as large as arising from $d_{in}$ 
and there is the option that these number of units be unconstrained (constraining only norms on composite weights as we shall see below).
The units on layer $\ell$ are indexed by $j_{\ell}$ in $\{1,2,\ldots,d_{\ell}\}$.
Each $W_{\ell}$ is the $d_{\ell-1} \times d_{\ell}$ matrix of weights and each matrix entry $w_{j_{\ell-1},j_{\ell}}=W_{\ell}[j_{\ell-1},j_{\ell}]$ is the weight between unit $j_{\ell-1}$ in layer $\ell$ and unit $j_{\ell}$ in layer $\ell$, where we drop the index specifying layer $\ell$ when it is clear from the indices $j_\ell$. Juxtapositions $W_\ell \phi$ and $W_{L} x$ denote the product of the weight matrices with the indicated vectors.  

We assume that each coordinate of the input vector $x$ has a bounded range which we take to be $[-1,1]$. There is the freedom that one of the coordinates of $x$ (say the last coordinate) is always assigned the value $-1$ so that by choice of the weights we are adjusting the location of the knot (sometimes called the bias or offset of the unit) as well as adjusting the coefficients of linear combinations of the non-constant coordinates of $x$. Likewise, we can arrange the network  so that the last entry of each row of $W_\ell$ multiplies a constant input, thus providing a freedom of offset of the internal activation functions.  The availability of a constant input node at each layer $\ell$ can be arranged by explicit reformulation or by specializing the present formulation, setting the last row of $W_\ell$ to be $0$, except in its last entry.

We focus on the case that $d_{out}=1$, though multidimensional outputs can be examined similarly.  In this case, there is but one output index $j_0=1$, and $W_1$ is a row vector of length $d_1$ with entries $w_{j_0,j_1}= w_{j_1}$. Accordingly, for networks of the first form, the function $f(W,x)$ is
$$\sum_{j_1} w_{j_1} \phi \big( \sum_{j_2} w_{j_1,j_2} \phi \big( \sum_{j_3} w_{j_2,j_3}  \cdots \phi \big(\sum_{j_L} w_{j_{L-1},j_L} x_{j_L}\big)  \big)\big).$$
Each unit computes $z_{j_\ell} = \phi(\sum_{j_{\ell+1}} w_{j_{\ell},j_{\ell+1}} z_{j_{\ell+1}})$, where  $z_{j_\ell}$ denotes the output value for unit $j_{\ell}$ on layer $\ell$, as a function of its inputs $z_{j_{\ell+1}}$, starting with the innermost layer, with the convention that $z_{j_L}$ is the input coordinate $x_{j_L}$.  

An important matter for the generality of representation is the freedom of negative and positive signs in this representation.  For each $j_{\ell}$ we need to be able to accomplish the same result as if an arbitrary subset of the $j_{\ell+1}$ contribute negative weights.  Instead of allowing the associated  $w_{j_{\ell},j_{\ell+1}}$ to be negative, we choose to consider the freedom to double the set of inputs from layer $\ell+1$ with the first $d_{\ell+1}$ being positive $z_{j_{\ell+1}}$ values and the second $d_{\ell+1}$ being the corresponding values with a minus sign attached to the activation.  Then, when a term in the sum is to be positive, it is by a selection from the first half and when a term is to be negative it is by selection from the second half.  

We analyze this setting cleanly by generalizing the meaning of $\phi$ for vector inputs to be a vector of twice the length with the positive part function for the first half and minus that for the second half.  Accordingly, the $d_{\ell}$ for $\ell=1,\ldots,L-1$ are taken to be twice their original values. Likewise, for the inputs, we modify $d_L$ to be $2 d_{in}$ and generalize the $z_{j_L}$ to match the $x_{j_L}$ for the first $d_{in}$ of these, and to match minus those values for the rest.

[For a network with twice the size at each layer to be strictly commensurate with a network of the original size with this sign-handling convention, it would require a duplication of weights interior to a node and its partner node of opposite sign on each layer $\ell$, as well as an understanding that for the links from layer $\ell$ at most one of each node and its partner node can be be non-zero.  The strictly commensurate networks are a subset of the networks considered here in which we do not make the limitations of the duplication, nor at most one partner non-zero.]

The point of the above arrangements is that we have facilitated analysis of the weights by probabilisitic methods, by having arranged the $w_{j_{\ell},j_{\ell+1}}$ to be nonnegative.

\section{Network and Subnetwork Variations} \label{sec:variation}
Critical to the analysis of the ramp networks is a homogeneity property. 
Namely, for nonnegative $w$ and a vector $z$ of even length, the vector $w \phi(z)$ equals $\phi(wz)$. Again this $\phi(wz)$ vector has the first half nonnegative and the second half nonpositive.  Applying this homogeneity repeatedly, if one wishes, and allowing weights on internal layers to be indexed by the output path, one can push all the weights to the innermost layer, such that
$$f(W,x) =  \sum_{j_1}  \phi \big( \sum_{j_2} \phi \big( \sum_{j_3}  \cdots \phi \big(\sum_{j_L} w_{j_1,j_2,\ldots,j_L} x_{j_L} \big) \big)\big),$$
with composite nonnegative weights 
$$w_{j_1,j_2,\ldots,j_L} = w_{j_1}  w_{j_1,j_2}  w_{j_2,j_3} 
\cdots 
w_{j_{L-1},j_L}.$$
This representation of $f(W,x)$ may be thought of as an unravelling of the graph of the network into a tree rooted at the output. It has nodes (at layer $\ell$ say) indexed by the sequence $j_1, j_2, \ldots, j_\ell$ marking the path from the root to this node.  At first glance this unravelled form may seem surprising. We find this form useful for analysis, even though the original form is more appropriate for function evaluation.
In the above equation, for each $j_1,\ldots,j_{\ell-1}$, the first half of each sum (over indices $j_{\ell}$) is for positive terms and the second half is for negative terms. It is valuable to take note that if a weight, $w_{j_1}$ say, is equal to zero, then the terms are made zero for all multi-indices $(j_1,j_2,\ldots,j_L)$ that share this $j_1$.  

The composite weight representation makes apparent some freedom of interlayer scaling of the weights that preserve the network function $f(W,x)$. 
In particular, at layer $\ell=1$, for any $j_1$ and positive $c_{j_1}$, if $w_{j_1}$ is multiplied by $c_{j_1}$ and $w_{j_1,j_2}$ is divided by the same $c_{j_1}$, for all $j_2$, then the composite weights are preserved. The same is true for any node $j_\ell$ with $1\!<\!\ell\! <\! L$ and positive value $c_{j_\ell}$, where for all $j_{\ell-1}$ and $j_{\ell+1}$ the $w_{j_{\ell-1},j_\ell}$ is multiplied by $c_{j_\ell}$ and the $w_{j_\ell,j_{\ell+1}}$ is divided by $c_{j_\ell}$.  If desired, such function-preserving modifications can be provided at the nodes in each layer in succession for $\ell$ from $1$ to $L-1$. As we shall see, this interlayer scaling can be used to balance total weights flowing into and out of nodes, which permits relationship between geometric and arithmetic forms of total weight variation.


There is freedom to have an auxiliary positive scalar weight $w_{j_0}=w_0$, which can be regarded as a output layer weight multiplying the final sum, so that the network functions $f(W,x)$ we study take the form
$$w_{j_0}\! \sum_{j_1} \!w_{j_1} \phi \big( \sum_{j_2}\! w_{j_1,j_2} \phi \big( \sum_{j_3}\! w_{j_2,j_3}  \cdots \phi \big(\sum_{j_L}\! w_{j_{L\!-\!1},j_L} x_{j_L}\big)\!  \big)\!\big),$$
or the same with $\phi_{out}$ applied thereto. The corresponding composite weights $w_{j_0,j_1,\ldots,j_L}$ are the same as the $w_{j_1,\ldots,j_L}$ above, now with multiplication by $w_{j_0}$. This $w_{j_0}$ provides no additional freedom of representation (as its action can be absorbed into $W_1$), yet it will allow some simplification of expression of our bounds, via a balance between $w_0$ and the total inner layer weights.  Recall that the output node index $j_0$ can take only one value since we are assuming $d_{out}=1$.

We are now set to provide some measures of size of the weights that appear in our approximation and risk bounds.

Set $V=V_L (W)$ to be the sum of products of the weights across the paths,
\begin{align}
V & = \sum_{j_0,j_1,\ldots,j_L} w_{j_0,j_1,\ldots,j_L} \nonumber \\
& = w_{j_0} \sum_{j_1,j_2,\ldots,j_L} w_{j_1}  w_{j_1,j_2}  w_{j_2,j_3}  
\cdots 
w_{j_{L-1},j_L}, \label{eq:totalvariation}
\end{align}
which we recognize as being equal to the entry-wise $\ell^1$ norm of the product of the weight matrices
$W_0 W_1 
\cdots 
W_{L}.$
We call this the \emph{variation} $V_L$ of the depth $L$ network $f(W,x)$. The variation is similar to the so-called path norm (a group-norm type regularizer \cite{kawaguchi2017, neyshabur2015, neyshabur2017geometry, neyshabur2015path}).

There will be a role for the variations of subnetworks, or \emph{subnetwork variations}. In particular, let $V_{j_\ell}^{out}=V_{j_\ell}^{out}(W)$, given by
$$V_{j_\ell}^{out} = w_{j_0} \sum_{j_{1}} \sum_{j_{2}} \cdots \sum_{j_{\ell-1}}w_{j_1}  w_{j_1,j_2}
\cdots 
w_{j_{\ell-1},j_{\ell}},$$ be the variation emanating out of node $j_\ell$ and let 
$V_\ell^{out} = \sum_{j_\ell}  V_{j_\ell}^{out}$ be the  
variation for the subnetwork of $d_{\ell}$ inputs that starts at layer $\ell$. Similarly, it is the entry-wise $ \ell^1 $ norm of $W_0 W_1 \cdots W_\ell $.

Likewise, let $V_{j_\ell}^{in}=V_{j_\ell}^{in}(W)$, given by
$$V_{j_\ell}^{in} = \sum_{j_{\ell+1}} \sum_{j_{\ell+2}} \cdots \sum_{j_{L}} w_{j_{\ell},j_{\ell+1}} w_{j_{\ell+1},j_{\ell+2}} 
\cdots 
w_{j_{L-1},j_L},$$ be the variation of the subnetwork that terminates at (flows into) node $j_\ell$ on layer $\ell$, and let $V_\ell^{in} = \sum_{j_\ell} V_{j_\ell}^{in}$, which is the entry-wise $\ell^1$ norm of the matrix product $W_{\ell+1} W_{\ell+2} \cdots W_L$.  

For $\ell=0$ we have $V_0^{out}=w_0$ and we see that the product of $V_0^{in}$ and $V_0^{out}$ is equal to $V$.  Furthermore, for each $\ell=0,1,\ldots,L\!-\!1$, the variation $V$ is a sum of products of these subnetwork variations 
$$V= \sum_{j_{\ell}} V_{j_\ell}^{out} V_{j_{\ell}}^{in}.$$
Similar expressions arise in our analysis. For each layer $\ell=0,1\ldots,L\!-\!1$, we have
a sum of geometric means of the subnetwork variations,
and its bound via the arithmetic mean
$$V_\ell = \frac{V_\ell^{out} + V_\ell^{in}}{2}.$$

The average across the layers of the subnetwork variations, denoted by $\overline V = \overline V(W)$, is given by
 $$\overline V = {\frac{1}{L}} \sum_{\ell=0}^{L-1} V_{\ell},$$
which we call the \emph{average variation}. Different scalings of the weight matrices result in various instantiations of the average variation, but we will postpone discussion of this until \prettyref{sec:scaling}. We similarly define $ \overline V^{out} $ and $ \overline V^{in} $ as the average of $ V^{out}_{\ell} $ and $ V^{in}_{\ell} $, respectively.

This average variation $\overline V$ is built simply by averaging the subnetwork variations $V_{\ell}^{out}$ and $V_{\ell}^{in}$ across the layers, so it built from the same sorts of objects as $V$ itself. Nevertheless, $V$ arises via products of subnetwork variations as described above, and a consequent intriguing aspect is that it is the square of the subnetwork value $\overline V$ that is naturally comparable to $V$.  Indeed, as we shall see, $\overline V^2$ exceeds $V$ and for our squared error bounds there will be a role for $v = \overline V\, \sqrt{V} $ which is comparable to, and exceeds, $V$. 


The input variations are related between layers by multiplication by $W_{\ell+1}$, yielding
$V_{j_\ell}^{in} =  \sum_{j_{\ell+1}} w_{j_{\ell},j_{\ell+1}} V_{j_{\ell+1}}^{in}$, and similarly for output variation $V_{j_\ell}^{out} =  \sum_{j_{\ell-1}}V_{j_{\ell-1}}^{out} w_{j_{\ell-1},j_{\ell}} $. 
There will be a role for a reduced input variation $V_{j_\ell}^{in,red} = \sum_{j_{\ell+1}\ne j_{\ell+1}^*} w_{j_{\ell},j_{\ell+1}} V_{j_{\ell+1}}^{in}$ in which the largest term in the sum is removed. We have corresponding values $V_{\ell}^{in,red}$, $V_\ell^{red}$, $ \overline V^{out, red} $, and $\overline V^{red}$, which are defined using $V_{j_\ell}^{in,red}$ in place of $V_{j_\ell}^{in}$.

Our risk and cover bounds are developed for arbitrary network functions of finite average variation $\overline V$ or $\overline V^{red}$, via composite variations defined by the products $v = \overline V\, \sqrt{V} $ and $v^{red} = \overline V^{red}\, \sqrt{V}$. The composite variation of the whole network is controlled if the sum of the constituent subnetwork variations is controlled. This structural condition fits with the perspective in engineering that, for a system to exhibit good behavior, each of its constituent components must also be operating effectively.

It will be seen, via interlayer scaling, that there is a canonical form of the network function in which, for $\ell\!=\!0,1,\ldots,L\!-\!1$, at each node $j_\ell$, the $V_{j_\ell}^{in}$ matches $V_{j_\ell}^{out}$, and hence $V_{\ell}^{in} = V_{\ell}^{out}=V_\ell$. This is analogous to conservation laws for electrical current or for volumes of fluid flow.  A similar canonical form for the reduced variation arranges for $V_{j_\ell}^{in,red}$ and $V_{j_\ell}^{out}$ to match. It can produce smaller bounds, though in this case, there is less of an analogy with conservation laws.

Arbitrary network functions have multiple representations, and, as we shall see, among such, the bounds involving $\overline V$ and $\overline V^{red}$, respectively, are optimized by choice of the respective canonical forms.

Another way of describing the average variation is that it is the entry-wise $ \ell^1 $ norm of the Ces\`aro average of successive matrix products.
In the canonical representation of the network weights, this expression has common value
$$
\left\|\frac{1}{L}\sum_{\ell=0}^{L-1}W_0W_1 
\cdots W_{\ell}\right\|_1 = \left\|\frac{1}{L}\sum_{\ell=0}^{L-1}W_{\ell+1}W_{\ell+2} \cdots W_L\right\|_1,
$$
where $ \|A\|_1 = \sum_{j_1,j_2}|a_{j_1,j_2}| $ for a matrix $ A $ with entries $ a_{j_1,j_2} = A[j_1,j_2] $.
As will be discussed further in \prettyref{sec:discussion} and in \prettyref{sec:examples}, it is such characterization via norms of matrix products rather than products of matrix norms that can yield favorable behaviour of $\overline V$ for possibly large $L$.

There is a more general notion of variation when $d_1,d_2,\ldots,d_{L-1}$ are free to be arbitarily large to achieve accurate approximation to a target function $f$.  In particular, for a specified distance metric and $d_L=2 d_{in}$, and for any function $f$ on $[-1,+1]^{d_{in}}$, we define its variation $V_L (f)$
as the infimum of numbers $V$ such that for every $\epsilon > 0$, no matter how small, there is a depth $L$ network of this variation $V$ with distance from $f$ not more than $\epsilon$.  

Likewise we take the average variation $\overline V_L(f)$ to be the infimum of $\overline V$ such that for every $\epsilon > 0$, there is a depth $L$ network of such average variation $\overline V$ (and arbitrarily large $d_1,d_2,\ldots,d_{L-1}$) with distance from $f$ not more than $\epsilon$.  Moreover, the composite variations $v(f)$ and $v^{red}(f)$, respectively, are defined similarly, as infima of $v = \overline V\, \sqrt{V} $ and $v^{red} = \overline V^{red}\, \sqrt{V} $ for which there are depth $L$ networks with such composite variations having arbitrarily small distance from $f$.

Let $a$ be normalized weights given by
$$a_{j_1,j_2,\ldots,j_L}=\frac{w_0}{V}\, w_{j_1,j_2,\ldots,j_L},$$
which can be interpreted as a joint probability distribution on the multi-indices $(j_1,j_2,\ldots,j_L)$, which is nonnegative and sums to $1$.

There are three reasons for calling the sum of composite weights $V=V_L$ in \prettyref{eq:totalvariation} the \emph{variation} of the network.  First, it is the total variation (in the probability theory sense) of the measure $W = V\, a$. 

Second, there is the calculus notion of the variation of functions of one variable, with respect to unit step functions, which is generalized in multiple dimensions to variation with respect to other classes (dictionaries) of bounded functions in 
\cite{barron1992}, also called the atomic norm of the function with respect to the dictionary 
\cite{chandrasekaran2012}, taken to be the infimum of sums of absolute values of weights of linear combinations of dictionary elements for arbitrarily accurate approximation of the specified function.

The function $f(W,x)=V f(a,x) = V \sum_{j_1} a_{j_1} f_{j_1} ( a,x)$, as seen below, is $V$ times a convex combination of functions in the dictionary of depth $L-1$ networks, of unit total weight. Accordingly, $V_L(f)$ is the variation of a function $f$ with respect to this dictionary and $\overline V_L(f)$ modifies it appropriately to take into account the functions arising at intermediate layers as well. These extend the notion of variation from \cite{barron1992,Barron1993,Barron1994} for single hidden-layer networks.  

Third, there is the simple notion of variation $\pm V_L$ as the potential range of a function.  As we will also see below, for $x$ in the unit cube $[-1,1]^{d_{in}}$ the function $f(W,x)$ satisfies $|f(W,x)| \le V_L.$  A set of input values (e.g. among the vertices of this unit cube) is said to saturate a network with specified weights $W$ if both $\pm V_L$ are achievable.  For networks that are strictly commensurate with signed weight networks of the original size, to be saturable, this entails achieving, for some such $x$, values of $f_{j_1}(a,x)$ that are all $+1$ and for other such $x$ values of $f_{j_1}(a,x)$ that are all $-1$, for all $j_1$ with $w_{j_1}>0$, which, in turn, is a saturable requirement of subnetworks.  Thus $V_L$ has an interpretation as the range of saturable networks. For example, if the input variables are all $1$ and if the weights are zero for the off-set and are non-zero only for the positive activation functions, then the value of $f(W,x)$ achieves the bound of $V_L$. The value $-V_L$ is similarly achievable.  Likewise, the subnetwork variations such as $V_{\ell}^{out}$ and $V_{j_\ell}^{in}$ share this interpretation of controlling the range of the corresponding subnetwork functions.

Having discussed quantities that measure the variation of a multi-layer network, let us now show how to construct and count the number of sparse networks that approximate any other deep network to a desired level of accuracy. We shall see that our notions of network and subnetwork variation play a crucial role in controlling the quality of the approximations.

\section{Constructing and Counting Sparse Approximants} \label{sec:cover}

Let us return our attention to the normalized weights $a_{j_1,j_2,\ldots,j_L}$ 
interpreted as a joint probability distribution on the multi-indices $(j_1,j_2,\ldots,j_L)$. Let $a_{j_{\ell},j_{\ell+1}}$ be the marginalized distribution for the pair $j_{\ell},j_{\ell+1}$ obtained by summing out over all the other $L-2$ indices. We note that this $a_{j_{\ell},j_{\ell+1}}$ is not the same as the normalized $w_{j_{\ell},j_{\ell+1}}$ because at least one of these $j_{\ell},j_{\ell+1}$ appear in other factors of the product representation of $W$. Let $a_{j_{\ell}}$ and $a_{j_{\ell+1}|j_{\ell}}$ be the associated marginal and conditional distributions for $j_{\ell}$ and for $j_{\ell+1}$ given $j_{\ell}$ induced by this joint distribution, where again we continue the slight abuse of notation as these distributions are in general not the same across $\ell=1,2,\ldots,L$.  

For any fixed $j_\ell$, the $w_{j_1,j_2,\ldots,j_L}$ factors as the product of 
$w_{j_1}  w_{j_1,j_2}
\cdots 
w_{j_{\ell-1},j_{\ell}}$ and 
$w_{j_{\ell},j_{\ell+1}} 
\cdots 
w_{j_{L-1},j_L}$ (a conditional independence). One implication is that the sum over all indices except $j_\ell$ likewise factors as the product of $V_{j_\ell}^{in}$ and $V_{j_\ell}^{out}$ and hence the marginals $a_{j_\ell}$ may be expressed in terms of the subnetwork variations as $$a_{j_\ell} = V_{j_\ell}^{out} V_{j_\ell}^{in} /V.$$ 

One sees that the joint distribution $a$ has a Markov structure (obtained from the product form of $W$), namely,
$$a_{j_1,j_2,\ldots,j_L} = a_{j_1} a_{j_2|j_1} a_{j_3|j_2} 
\cdots 
a_{j_L|j_{L-1}}.$$
The interpretation is that with respect to these weights the $j_1,j_2,\ldots,j_L$ forms a Markov chain (with inhomogeneous transitions), here interpreted as starting at the nearly outermost index $j_1$ and ending at the innermost index $j_L$. Of course, it is also a Markov chain starting at a random innermost index and transitioning forward toward the output, in accordance with a generative model, but we will not use that here.  In accordance with this representation, we may use the homogeneity of the ramp and write
$$f(W,x) = V \,f(a,x),$$
where
$$f(a,x) =  \sum_{j_1}  \phi \big( \sum_{j_2} \phi \big( \sum_{j_3}  \cdots \phi \big(\sum_{j_L} a_{j_1,j_2,\ldots,j_L} x_{j_L} \big) 
\big)\big)$$
is seen to also have the representation
$$\sum_{j_1} a_{j_1} \phi \big( \sum_{j_2} a_{j_2|j_1} \phi \big( \sum_{j_3} a_{j_3|j_2}  \cdots \phi \big(\sum_{j_L} a_{j_L|j_{L-1}} x_{j_L}\big) 
\big)\big).$$
That is, starting from the representation with weight matrices $W_1,W_2,\ldots, W_{L}$ and using homogeneity of the ramp, and decomposition of the normalized product representation, we obtain an iterated expectation representation, interspersed with the nonlinearities.  It is this representation that facilitates our probabilistic method of analysis.

In keeping with our previous notation, the subnetwork output $z_{j_\ell}$ at any specified node $j_\ell$ is $z_{j_\ell}= f_{j_\ell}(a,x)$ given by 
$$\phi \big( \sum_{j_{\ell+1}} a_{j_{\ell+1}|j_\ell} \phi \big( \cdots \phi \big(\sum_{j_L} a_{j_L|j_{L-1}} x_{j_L}\big)\big)\big).$$
For $x$ in $[-1,1]^{d_{in}}$ the Lipschitz($1$) property of the ramp functions (with either the positive or negative sign) and the fact that the $a$ sums to $1$ leads inductively to the property that $f(a,x)$ is in the interval $[-1,1]$ and each subnetwork function $f_{j_\ell}(a,x)$ in this representation is also bounded by $1$.

This representation is fundamental to our analysis of the size of the class of such networks, that determines their statistical learnability, subject to entry-wise $\ell^1$ constraints on products of the weight matrices.  We use metric entropy to quantify the size of this class.  Namely, we specify a discrete set of parameters $\tilde a$, providing a cover, determine the cardinality of this set of networks, and bound the accuracy with which any $f(a,x)$ is represented by some $f(\tilde a , x)$. The accuracy will be quantified by a bound that holds, in particular, for all empirical $\mathbb{L}^2$ norms on any finite set of data.

We use a trick in which, for any $a$, we draw representer parameters $\tilde a$ at random from a finite set we specify and show that the bound holds for the expectation, so accordingly there exists a representer of that accuracy.  The representer set (the cover) is indexed by a parameter $M$ that controls both the accuracy and the cardinality. 

The representers are built from the collection of vectors of nonnegative integers of specified sum
$\underline K =(K_{j_1,j_2,\ldots,j_L}: \sum_{j_1,j_2,\ldots,j_L} K_{j_1,j_2,\ldots,j_L} = M, \quad K_ {j_1,j_2,\ldots,j_L} \ge 0)$.
The index set is of size 
$D= D_L = d_1 d_2
\cdots
d_L$, which is equal to $d^L$, where $d= (d_1 d_2 \cdots d_L)^{1/L}$ is the geometric mean of the $d_\ell$.  [This geometric mean reduces to the common value of $d_1,d_2,\ldots,d_L$ when they are all equal.]
By the stars and bars argument of Feller \cite[page 38]{feller1971}, the number of such nonnegative integer vectors of specified sum $M$ equals $ {M+D-1} \choose {M} $, which will bound the cardinality of our cover.  It has the familiar bounds of $(M+1)^{D}$ and $D^M = d^{LM}$, 
the latter 
providing a log-cardinality bound of $M\log D = LM\log d$.  

There is also the entropic bound that the binary logarithm of $ {M+D-1} \choose {M} $ is less than $(M+D-1) H(M/(M+D-1))$, where $H$ is the binary entropy function. 
This inequality is familiar in information theory, see 
\cite[Lemma 8, Page 310]{macwilliams1977}, and is a direct consequence of the Chernoff bound for $ \Bern(1/2) $ trials, i.e., $ \prob{\Bin(m, k/m) \leq k} \leq 2^{-mD(k/m || 1/2)} $, where $ D(k/m || 1/2) $ is the Kullback-Leibler divergence (in bits) between $ \Bern(k/m) $ and $ \Bern(1/2) $.  
Note that $M \log(e(1+(D-1)/M))$ is a further upper bound of $(M+D-1) H(M/(M+D-1))$, to which it is close when $D$ is large compared to $M$.  With $M$ appearing in the denominator this bound is superior to the $M \log D$ bound for $M > e$.  We also have the sometimes more convenient further bound $M \log(2eD/M)$, valid for $D \ge M-1$, again superior to $M \log D$, now for $M > 2e$.

We will not use fully these vectors of integers, but only the integers $K_{j_{\ell},j_{\ell+1}}$ associated with indices $(j_{\ell},j_{\ell+1})$ that arise by summing out over the other indices.  There is a possibility to find a useful  reduction of the log cover-size by considering the collection of such doubly indexed counts. 

The previous results are useful when the dimensions $ d_1, d_2, d_3, \dots, d_{L-1} $ are not excessively large. Going further, an ideal theory would allow for these dimensions to be arbitrary. Thankfully, this can be done as we will now discuss.

\section{Cardinality Improvement} \label{sec:improve}

For the cover from \prettyref{sec:cover}, there is a reduction in the log cover size accounting that is available for large $d_1,d_2,\ldots,d_{L-1}$.  In particular we provide improvement when $d_\ell > 2M$.  In the present formulation, at layer $\ell < L$, each node appears with marginal weight $\tilde a_{j_\ell} = K_{j_\ell}/M$, with at most $M$ of the first half and second half of the nodes having non-zero weight. If $d_\ell > 2M$, then without changing the function $f(\tilde a,x)$, we may eliminate $d_\ell -2M$ of the zero weight nodes on layer $\ell$ and reindex those that remain, with a fixed schematic of $M$ with nonnegative activation and $M$ with nonpositive activation.   So the same functions in the cover are implemented with $d_\ell^{new} = \min\{d_\ell,2M\}$.  Now $d_L^{new}=d_L=2 d_{in}$ is unchanged because at the innermost layer we must preserve the identity of the original input coordinates.  The new number of multi-indices is $D^{new} = \prod_{\ell=1}^L d_\ell^{new}$ which can replace $D$ in the cardinality bounds from \prettyref{sec:cover}. 

The products of minima are not more than the minimum of the products.  Thus $D^{new} \le 2M [\min\{\bar d,2M\}]^{L-2} 2 d_{in}$ where $\bar d$ is the geometric mean of $d_2,d_3,\ldots,d_{L-1}$. [In this bound we have replaced the first factor $\min\{d_1,2M\}$ with the bound $2M$ so that there is a convenient factor $M$ to cancel inside the log in the $M \log(2eD^{new}/M)$ bound below.] 
As before, the log-cardinality of counts with sum $M$ is not more than $M \log(e(1+(D^{new}-1)/M))$. This is further bounded by $M \log(2eD^{new}/M))$. 
Thus, in summary, we have shown that the log cardinality of such multi-layer networks is at most
$$
(L\!-\!2)M \log(\min\{\bar d, 2M\}) \,+\, M \log(8e \,d_{in}).
$$
In particular, no matter how large $d_1$ and $\bar d$ are, the log-cardinality of the set of such networks is at most
$$
(L\!-\!2)M \log(2M) \,+\, M \log(8e \,d_{in}).
$$
A minor matter concerns covering networks for which the variation $V_L (W)$ can be less than a specified value $V$.  In this case, the measure $a$ in the preceding section is a sub-probability.  Fill it out to a probability by arranging a null index for $j_1$, with trivial subnetwork $f_0(x)=0$, and assigning it weight on the outer layer $1$ of $a_0 = 1- V_L(W)/V$.  This increases $d_1$ by one, but does not increase the bound on the size of our cover, which is independent of $d_1$.  Similar arrangements can be made for covering networks for which subnetwork variations are less than specified values.

\section{Main Result} \label{sec:main}
With this setup and notation, our main result is as follows.

\begin{theorem} \label{thm:main}
Consider the parameterized family $ \calF_{L,v} $ of depth $L$ network functions with composite variation $\overline V\,\sqrt{V}$
at most $v$.  
There is a subfamily $ \widetilde\calF_{L, v}$ of log-cardinality at most 
\begin{equation} \label{eq:logcard}
(\!L-\!2)M \log(\min\{\bar d, 2M\}) + M \log(8e \, d_{in}),
\end{equation} 
such that for any probability measure $P$ on $[-1,1]^{d_{in}}$ and any $f(W,x) $ belonging to $ \calF_{L,v}$, there is a sparse approximant $ f(\widetilde{W},x) $ in $ \widetilde\calF_{L, v} $  such that
\begin{equation}\label{eq:bound2}\int   |f(W,x) - f(\widetilde{W},x)|^2 P(dx)\leq \left[\frac {L v}{\sqrt{M}}\right]^2.\end{equation}
Likewise, for networks with composite reduced variation $ \overline V^{red}\, \sqrt{V} $ at most $v^{red}$, we have the same log-cardinality of cover \prettyref{eq:logcard}, and now the accuracy bound
\begin{equation}\label{eq:bound3}\int   |f(W,x) - f(\widetilde{W},x)|^2 P(dx)\leq \left[\frac {2Lv^{red}}{\sqrt{M}}\right]^2.\end{equation}
\end{theorem}

Next we state a refinement from which those two approximation bounds are derived. For each $W$ there is a $\widetilde W$ in the cover of the given cardinality bound for which
$$\int   |f(W,x) - f(\widetilde{W},x)|^2 P(dx) \quad\quad\quad $$
\begin{equation} \label{eq:refinedbound}
 \leq \frac {V}{M} \left[\sum_{\ell=0}^{L-1} \sum_{j_\ell} \sqrt{V^{out}_{j_\ell} V^{in}_{j_\ell}} \sigma_{j_\ell} \right]^2, 
\end{equation}
where the $\sigma_{j_0}^2= \sigma_{j_0,W}^2$ is the $P$ expectation of the variance of $z_{J_1}(a,x)$ with respect to the distribution $a_{j_1}$ for $J_1$ in $\{1,\ldots,d_1\}$, the $\sigma_{j_\ell}^2=\sigma_{j_\ell,W}^2$ is the $P$ expectation of the variance of $z_{J_{\ell+1}}(a,x)$ with respect to the conditional distribution $a_{j_{\ell+1}|j_\ell}$ for $J_{\ell+1}$ given $j_\ell$, and the $V=V(W)$, $V^{in}_{j_\ell}=V^{in}_{j_\ell}(W)$ and $V^{out}_{j_\ell}=V^{out}_{j_\ell}(W)$ are, respectively, the full network and input and output subnetwork variations determined by the weights $W$.

The $|z_{j_\ell}|$ are not more than $1$, so that  $\sigma_{j_\ell}\le 1$ for all $j_\ell$, and the right side of \prettyref{eq:refinedbound} may be replaced by
\begin{equation} \label{eq:refinedboundSimple}
\frac {V}{M} \left[\sum_{\ell=0}^{L-1} \sum_{j_\ell} \sqrt{V^{out}_{j_\ell} V^{in}_{j_\ell}} \right]^2.
\end{equation}
The first bound \prettyref{eq:bound2} follows from the refinement by using the geometric mean inequality that $\sqrt{V^{out}_{j_\ell} V^{in}_{j_\ell}}$ is not more than $(1/2) \big(V^{out}_{j_\ell} + V^{in}_{j_\ell}\big)$. 

Likewise, the second bound \prettyref{eq:bound3} follows from the refinement using $$\sigma_{j_\ell}^2 \le \int \sum_{j_{\ell+1}} a_{j_{\ell+1}|j_\ell} \big( z_{j_{\ell+1}}(a,x)-z_{j^*_{\ell+1}}(a,x)\big)^2 P(dx),$$ where for each $j_\ell$ the $j_{\ell+1}^*$ is a nonrandom choice (depending on $j_\ell$) to which it is linked. A natural choice is the one maximizing $w_{j_\ell,j_{\ell+1}} V_{j_\ell}^{in}$, or equivalently maximizing $a_{j_{\ell+1}|j_\ell}$. This nullifies the contribution to the input variation from this highest weight link to $j_\ell$ which comes from $j_{\ell+1}^*$. It leads to $V^{in,red}_{j_\ell}$ in place of $V^{in}_{j_\ell}$.  Here we are using that $|z_{j_{\ell+1}}-z_{j^*_{\ell+1}}|$ is zero at $j_{\ell+1}= j^*_{\ell+1}$ and is not more than $2$ otherwise. It leads to the bound
\begin{equation} \label{eq:refinedboundReduced}
\frac {V}{M} \left[ 2 \sum_{\ell=0}^{L-1} \sum_{j_\ell} \sqrt{V^{out}_{j_\ell} V^{in,red}_{j_\ell}} \right]^2, 
\end{equation}
with $V^{in,red}_{j_\ell}$ in place of $V^{in}_{j_\ell}$, from which one obtains \prettyref{eq:bound3}, again by using the geometric mean inequality.




As a corollary to this theorem, we provide a quantitative bound on the complexity of the function class $ \calF_{L} $. First, we require a couple of definitions. 
\begin{definition}
Let $ P $ be a probability measure on a measurable space and suppose $ \calF $ is a family of functions in $ \mathbb{L}^2(P) $. A subfamily $ \widetilde{\calF} $ is called an $ \epsilon $-net for $ \calF $ if for any $ f \in \calF $, there exists $ \widetilde{f} \in \widetilde{\calF} $ such that $ \|f - \widetilde{f}\| \leq \epsilon $. The logarithm of the minimum cardinality of $ \epsilon$-nets is called the covering $ \epsilon $-entropy of $ \calF $ and is denoted by $ \cal\calV_{\calF}(\epsilon) $.
 
\end{definition}

One immediate corollary of \prettyref{thm:main} is a bound on $ \calV_{\calF_{L, v}}(\epsilon) $, which we state next.
\begin{corollary}
The covering $ \epsilon $-entropy of $ \calF_{L, v} $ is at most
\begin{equation} \label{eq:cover1}
\frac{L^2 v^2}{\epsilon^2}\left[(L-2)\log(\min\{\bar d, 2L^2 v^2/\epsilon^2)\}) + \log(8e \,d_{in})\right].
\end{equation}
\end{corollary}

\section{Proof of \prettyref{thm:main}} \label{sec:proof}
The randomization argument we use to bound the accuracy of approximation is as follows.  For any specified normalized network parameters $a$ (which accordingly has the Markov interpretable structure), Let  $\underline K= (K_{j_1,j_2,\ldots,j_L})$ be distributed as a Multinomial$(M,a)$.  [If $a$ is a subprobability the remaining mass is placed a on null-sequence with associated $z_{j_1}$ set to be $0$ when $j_1=0$. Zero values of $j_\ell$ for $\ell > 1$ only co-occur with $j_1=0$ and are immaterial to the accounting and to the function evaluation.] Sum out over unspecified indices to form pairwise counts $K_{j_{\ell},j_{\ell+1}}$. We know that, jointly over choices of $j_{\ell},j_{\ell+1}$, these have the Multinomial$(M, (a_{j_{\ell},j_{\ell+1}}: (j_{\ell},j_{\ell+1}) \in \{1,\ldots,d_{\ell}\} \times \{1,\ldots,d_{\ell+1}\}))$ distribution.  [In the subprobability case the index set for $j_\ell$ is $\{0,1,\ldots,d_\ell\}$ starting at $0$.] 

We could form $\tilde a = \underline {K}/M$ as the representer of $a$.  This would be the empirical distribution (from relative frequencies) corresponding to the counts $\underline K$ that arise when we draw $M$ independent multi-indices $\underline {J}(m)=(J_1(m),J_2(m),\ldots,J_L(m))$ from the distribution $a$. Though drawn from the Markov $a$, such a joint distribution $\tilde a$ on $L$ indices would not necessarily have the Markov structure.  

Instead, we form $\tilde a$ as the Markov measure on $(j_1,j_2,...j_L)$ consistent with the pairwise marginals $\tilde a_{j_{\ell},j_{\ell+1}}= K_{j_{\ell},j_{\ell+1}}/M$. These have individual marginals $\tilde a_{j_{\ell}}= K_{j_{\ell}}/M$, and, conditional distributions $\tilde a_{j_{\ell+1}|j_{\ell}}$ which are defined as $K_{j_{\ell},j_{\ell+1}}/K_{j_{\ell}}$ when $K_{j_{\ell}}>0$ (and as $0/0 = 0$ otherwise).

Accordingly, we have $ f(\tilde a,x) $ is equal to
$$ 
\sum_{j_1} \tilde a_{j_1} \phi \big( \sum_{j_2} \tilde a_{j_2|j_1} \phi \big( \sum_{j_3} \tilde a_{j_3|j_2}  \cdots \phi \big(\sum_{j_L} \tilde a_{j_L|j_{L-1}} x_{j_L} \big) 
\big)\big).$$
A key step in our analysis is to form the hybrid functions $f_\ell (\tilde a, a,x)$ given by
$$\sum_{j_1} \tilde a_{j_1} \phi \big( \sum_{j_2} \tilde a_{j_2|j_1} \phi \big(\cdots \sum_{j_{\ell}} \tilde a_{j_{\ell} |j_{\ell-1}} \phi \big(\sum_{j_{\ell+1}}$$
$$ a_{j_{\ell+1}|j_\ell} \cdots \phi \big(\sum_{j_L} \ a_{j_L|j_{L-1}} x_{j_L} \big) 
\big)\big),$$
which uses $\tilde a$ on the $\ell$ outermost layers and nonrandom $a$ on the $L-\ell$ innermost layers. When $\ell=1$ the tilde is only on the $\tilde a_{j_1}$ values on the outermost layer. When $\ell=0$ we take this function to be $f_0(\tilde a,a,x)=f(a,x)$ and when $\ell=L$ we take it to be $f_L(\tilde a,a,x)=f(\tilde a, x)$.

The error between $f(\tilde a,x)$ and $f(a,x)$ is written as a telescoping sum of (successively collapsing) differences
$$f(\tilde a, x) - f(a,x) = \sum_{\ell=0}^{L-1} [f_{\ell+1}(\tilde a,a,x)- f_{\ell}(\tilde a,a,x)]$$
in which the $f_{\ell+1}(\tilde a,a,x)$ and  $f_{\ell}(\tilde a,a,x))$ differ only on layer $\ell+1$, the former using $\tilde a_{j_{\ell+1}|j_{\ell}}$ and the later using $a_{j_{\ell+1}|j_{\ell}}$.  
The squared error $|f(\tilde a,x) - f(a,x)|^2$ is equal to 
$$\sum_{\ell,\ell'} [f_{\ell+1}(\tilde a,a,x)- f_{\ell}(\tilde a,a,x)][f_{\ell'+1}(\tilde a,a,x)- f_{\ell'}(\tilde a,a,x)].$$
For any distribution $P$ for $X$ in the cube $[-1,1]^{d_{in}}$, including, for instance, an empirical distribution $P=P_n$ on any $n$ data points $X_1,X_2,\ldots,X_n$ in this cube, we bound the $\mathbb{L}^2(P)$ norm square $\| f(\tilde a,\cdot ) - f(a,\cdot )\|^2 = \int   |f(\tilde a,x) - f(a,x)|^2 P(dx)$, by taking the integral inside the above double sum.

To show there is a choice $\tilde a$ producing a function with desirably small distance from the function at $a$, we consider our random $\tilde a$ and bound the expected squared distance $E\big[\int |f(\tilde a,x) - f(a,x)|^2 P(dx)]$ by also bringing the expectation inside the double sum and using the Cauchy-Schwarz inequality. Thus we obtain the bound
$$E\big[\int |f(\tilde a,x) - f(a,x)|^2 P(dx)] \le \sum_{\ell,\ell'} r_\ell r_\ell' = \big(\sum_{\ell} r_\ell \big)^2,$$
where 
$$r_\ell^2 =  \int E\left[|f_{\ell+1}(\tilde a,a,x)- f_{\ell}(\tilde a,a,x)|^2\right] P(dx).$$
The inequality
$$\left[ E\big[\int |f(\tilde a,x) - f(a,x)|^2 P(dx)] \right]^{1/2} \le  \sum_{\ell} r_\ell $$
may also be seen as a use of the triangle inequality for the indicated $\mathbb{L}^2$ norm (using joint expectation with respect to the distribution of $\tilde a$ and $x$), applied to the telescoping decomposition.

Return attention to the ingredients $r_\ell^2$ given above as an integral.
Inside the integral is the expected square of the difference $f_{\ell+1}(\tilde a,a,x)- f_{\ell}(\tilde a,a,x)$ using the distribution of $\tilde a$ for fixed $x$.

When $\ell=0$ the difference takes a simplified form, with $|f_{1}(\tilde a,a,x)- f_{0}(\tilde a,a,x)|$ being equal to
$$| \sum_{j_1} (\tilde a_{j_1}  - a_{j_1} ) z_{j_1}|,$$
where $z_{j_1}=f_{j_1}(a,x)$ is the output of the subnetwork terminating at the outermost-layer node of index $j_1$, where the parameters from its inner layers are from $a$ and the input is $x$. For notational conciseness in writing $z_{j_1}$, we suppress the dependence on $a$ and $x$. If there is a Lipschitz(1) activation applied to the final output, then the same expression arises as a pointwise upper bound on the absolute difference $|f_{1}(\tilde a,a,x)- f_{0}(\tilde a,a,x)|$. 

The sum in this expression is interpretable as a mean-zero average of functions arising in sampling $J_1(m)$ for $m=1,\dots,M$. Indeed, this expression is the same as
\begin{align*}
& \left|\frac{1}{M} \sum_{m=1}^M \big[ \sum_{j_1} (\indc_{\{J_1(m)=j_1\}} - a_{j_1} ) z_{j_1}\big]\right| = \\ &  \qquad\qquad\qquad\qquad\qquad\qquad \left|\frac{1}{M} \sum_{m=1}^M (z_{J_1(m)}- \mu)\right|,
\end{align*}
where $\mu = \sum_{j_1} a_{j_1} z_{j_1}$, the mean of $z_{J_1}$, is merely the value of the function $f(a,x)$.  Indeed, putting in the dependence on $x$, this $\mu$ is $\sum_{j_1} a_{j_1} f_{j_1}(a,x)$.
Accordingly, the expected square of this expression, for any fixed $x$, is simply $\mbox{VAR}(z_{J_1})/{M}= \sigma_{j_0}^2(x)/M$
where $\mbox{VAR}(z_{J_1} )=  \sigma_{j_0}^2(x)=\sum_{j_1} a_{j_1}  (z_{j_1}- \mu)^2$. It has expectation $\sigma_{j_0}^2/M$ for $x$ drawn from $P$. Consequently, $r_0 \le \sigma_{j_0} /M^{1/2}$. 

The distribution $a_{j_1}$ in these expressions is the same as $a_{j_1|j_0}$, recalling that conditioning on $j_0$ is degenerate as there is only one outermost node.  
Use of the subscripts $j_0$ allows the notation for this bound for the outer layer to blend well with what we obtain for the inner layers.

The variance $\sigma_{j_0}^2(x)$ is not more than $1$, since the $z_{j_1}$ are bounded by $1$, uniformly in $x$, which is used in our first bound. It is also not more than  the expected square with $\mu$ replaced by any value (any function of $x$) not depending on $j_1$. One natural choice is
$z_{j_1^*}$ where $j_1^*$ is the index with the largest value of $a_{j_1}$.  This produces the bound $\sigma_{j_0}^2(x)\le 4 \sum_{j_1\ne j_1^*} a_{j_1}$, uniformly in $x$.  In terms of the original weights this yields $\sigma_{j_0}^2 \le  4 V_0^{in,red}/V_0^{in}$, as used in our second bound, where $V_0^{in,red} = \sum_{j_1\ne j_1^*} w_{j_1} V_{j_1}^{in} = \|W_1^*W_2 \cdots W_{L}\|_1 $ is the reduced input variation, in which $W_1^*$ is the row vector of $w_{j_1}$, with nullification of the entry $j_1^*$ which would otherwise provide the largest contribution.

The cases of $\ell=1,\dots,L-1$ have some similarity of analysis, with an interesting twist. Repeatedly using the triangle inequality (moving absolute values inside sums) and the Lipschitz property of $\phi$, one finds that the difference between $f_{\ell+1}(\tilde a,x)$ and  $f_{\ell}(\tilde a,a,x)$ has the following pointwise bound on its absolute value
$$\sum_{j_1,\ldots, j_{\ell-1}} \tilde a_{j_1,\ldots, j_{\ell-1}} \,| \sum_{j_\ell} \tilde a_{j_\ell|j_{\ell-1}} \big( \phi(\tilde z_{j_\ell}^{in}) -\phi(z_{j_\ell}^{in})\big)|,$$
where $\tilde z_{j_\ell}^{in}$ is $\sum_{j_{\ell+1}} \tilde a_{j_{\ell+1}|j_{\ell}}z_{j_{\ell+1}}$ and its counterpart $z_{j_\ell}^{in}$ is $\sum_{j_{\ell+1}} a_{j_{\ell+1}|j_{\ell}}z_{j_{\ell+1}}$, where $z_{j_{\ell+1}}=f_{j_{\ell+1}}(a,x)$ is the output of the indicated intermediate-layer unit when the parameters from its inner layers are from $a$ and the input is $x$.
The right side of this pointwise bound simplifies by marginalization to 
$$| \sum_{j_\ell} \tilde a_{j_\ell} \big( \phi(\tilde z_{j_\ell}^{in}) -\phi(z_{j_\ell}^{in})\big)|.$$
The $\tilde z_{j_\ell}^{in}$ 
as expressed above is a random quantity arising from the counts $K_{j_{\ell},j_{\ell+1}}$. When conditioning on $K_{j_\ell}$, these counts $K_{j_{\ell},j_{\ell+1}}$ for $j_{\ell+1}$ in $\{1,\ldots,d_{\ell+1}\}$ have a joint multinomial distribution of size $K_{j_\ell}$ and parameters ($a_{j_{\ell+1}|j_{\ell}},j_{\ell+1} \in \{1,\ldots,d_{\ell+1}\}$). Accordingly, the $\tilde z_{j_\ell}^{in}$ have conditional mean $z_{j_\ell}^{in}$. 

[If the conditional mean of the $\phi(\tilde z_{j_\ell}^{in})$ were equal to $\phi(z_{j_\ell}^{in})$ then we could take the expected square, nullify cross-product terms, and obtain a better risk bound of $L^2 V^2/M$.  But the positive part activation function is convex (and minus the positive part is concave) so these conditional expectations are greater or less than the target $\phi(z_{j_\ell}^{in})$ whenever the conditional distribution of $\tilde z_{j_\ell}^{in}$ straddles $0$ (the point of nonlinearity of $\phi$).  This is a conditional bias. It is an open question whether a sampling strategy can be arranged to remove or control this statistical bias in such a way that the expectations of these cross-terms become zero. In the absence of such refined understanding, we proceed instead as follows.] 

Apply one more use of the triangle inequality and the Lipschitz property of $\phi$ to bound $|f_{\ell+1}(\tilde a,x)-f_{\ell}(\tilde a,a,x)|$ by the  following expression
$$ \sum_{j_\ell} \tilde a_{j_\ell} | \tilde z_{j_\ell}^{in} - z_{j_\ell}^{in}|= \sum_{j_\ell} \tilde a_{j_\ell} \, |U_{j_\ell}|$$
where $U_{j_\ell}=\tilde z_{j_\ell}^{in} - z_{j_\ell}^{in}$.
It is the conditional empirical average $\sum \tilde a_{j_{\ell+1}|j_{\ell}} z_{j_{\ell+1}}$ minus its conditional mean $z_{j_\ell}^{in}= \sum  a_{j_{\ell+1}|j_{\ell}} z_{j_{\ell+1}}$ and so, for each $x$, it is equally well expressed as 
$$U_{j_\ell} = \sum_{j_{\ell+1}} \tilde a_{j_{\ell+1}|j_{\ell}}  (z_{j_{\ell+1}} - z_{j_\ell}^{in}).$$
Thus, by the same reasoning as above, the difference $U_{j_\ell}$ is an average of $K_{j_\ell}$ conditionally independent copies of the random variables $z_{J_{\ell+1}} - z_{j_\ell}$ of conditional mean zero and conditional variance $\sigma_{j_\ell}^2(x)$, and hence it has conditional expected square equal to $\sigma_{j_\ell}^2(x) /K_{j_\ell}$.

Per the expression above, we bound the expected square of $|f_{\ell+1}(\tilde a,a,x)- f_{\ell}(\tilde a,a,x)|$, for any given $x$, by the following
$$E\big[(\sum_{j_\ell} \tilde a_{j_\ell} \, |U_{j_\ell}|)^2\big],$$
which is
$$\sum_{j_\ell,j'_\ell} E\left[ \tilde a_{j_\ell} \tilde a_{j'_\ell}\, |U_{j_\ell}|\,|U_{j'_\ell}|\right].$$
We examine the expectation for each $j_{\ell}$ and $j'_{\ell}$ in $\{1,\ldots,d_{\ell}\}$ by iterated expectation, conditioning on the values of $K_{j_\ell}$ and $K_{j_\ell}$
\begin{equation} \label{eq:conditional}
E\left[ \tilde a_{j_\ell} \tilde a_{j'_\ell}\, E\left[|U_{j_\ell}|\,|U_{j'_\ell}|\, {\big|} \, K_{j_\ell},K_{j'_\ell}\right]\right].
\end{equation}
Recognize that the terms with $j_\ell = j'_\ell$ involve the conditional expected square $E\big[(U_{j_\ell})^2 \big|K_{j_\ell}\big]$. 

For the terms with $j'_\ell$ distinct from $j_\ell$, we use that the counts $K_{j_\ell,j_{\ell+1}}$ determining $U_{j_\ell}$ are conditionally independent of the corresponding counts determining $U_{j'_\ell}$,  conditional on the sums $K_{j_\ell}$ and $K'_{j_\ell}$. Hence the expected product factors into the product of expectations via
$$
E\left[|U_{j_\ell}|\,|U_{j'_\ell}|\, {\big|} \, K_{j_\ell},K_{j'_\ell}\right] = E\big[|U_{j_\ell}|\, {\big|} \, K_{j_\ell}\big]E\big[|U_{j'_\ell}|\, {\big|} \, K_{j'_\ell}\big].
$$

Along with the Cauchy-Schwarz inequality applied conditionally, this allows us to bound \prettyref{eq:conditional} by
$$E\left[ \tilde a_{j_\ell} \tilde a_{j'_\ell}\, \sqrt{E[(U_{j_\ell})^2 \big|K_{j_\ell}] \,E[(U_{j'_\ell})^2\,\big|K_{j'_\ell}]}\,\right].$$

As argued above, the $\sqrt{E[(U_{j_\ell})^2 |K_{j_\ell}]}\,$ and its counterpart at $j'_\ell$ can be evaluated and are equal to $\sigma_{j_\ell}(x)/\sqrt{K_{j_\ell}}\,$ and $\sigma_{j'_\ell}(x)/\sqrt{K_{j'_\ell}}\,$ respectively, for each $x$. Thus, using the form of $\tilde a_{j_{\ell}} = K_{j_\ell}/M$, the iterated expectation has the following bound, for each $x$,
\begin{equation} \label{eq:cross}
E\left[ \tilde a_{j_\ell} \tilde a_{j'_\ell}\,  \frac{\sigma_{j_\ell}(x) \sigma_{j'_\ell}(x)} {\sqrt{K_{j_\ell}K_{j'_\ell}}}\,\right] \!=\! \frac {\sigma_{j_\ell}(x) \sigma_{j'_\ell}(x)}{M^2} E\left[\sqrt{K_{j_\ell} K_{j'_\ell}}\,\right],
\end{equation}
which, by another application of the Cauchy-Schwarz inequality\footnote{Using Jensen's inequality and the formula for the covariance between two marginals of a multinomial distribution $ \mbox{COV}(K_{j_\ell},K_{j'_\ell}) = -Ma_{j_\ell}a_{j'_\ell} $\;,
this bound can be improved by a factor of $\sqrt{(M-1)/M}$.},
is bounded by 
$$\frac{\sigma_{j_\ell}(x) \sigma_{j'_\ell}(x) \,\sqrt {a_{j_\ell}}\, \sqrt{a_{j'_\ell}}}{M}.$$
Taking the expectation with respect to $x$ drawn according to $P$, and using Cauchy-Schwarz yet again yields 
$$\frac{\sigma_{j_\ell} \sigma_{j'_\ell} \,\sqrt {a_{j_\ell}}\, \sqrt{a_{j'_\ell}}}{M}.$$
Take the sum over $j_\ell,j'_\ell$ to obtain the bound
on $r_\ell^2$ of
\begin{equation}\label{eq:ineqa}
\frac{1}{M}\big( \sum_{j_\ell} \sigma_{j_\ell} \sqrt{a_{j_\ell}}\,\big)^2.
\end{equation}
Thus for $\ell=1,\ldots,L-1$ we have
$$r_\ell \le \frac{1}{M^{1/2}}  \big( \sum_{j_\ell} \sigma_{j_\ell} \sqrt{a_{j_\ell}}\,\big).$$
Including the simpler $\sigma_{j_0} /M^{1/2}$ bound for the $\ell=0$ case, we have that the expectation of $\int |f(\tilde a,x)-f(a,x)|^2 \,P(dx)$ is not more than
$$\frac {1}{M} \left[\sigma_{j_0}\,+\, \sum_{\ell=1}^{L-1} \big( \sum_{j_\ell} \sigma_{j_\ell} \sqrt{a_{j_\ell}}\,\big)\right]^2.$$
The minimum over choices of $\tilde a$ is always less than or equal to the expectation.
So for for any normalized weights $a$ there must be a representer $\tilde a$, and hence for any normalized $L$ layer network function $f(a,x)$, there is a representer $f(\tilde a,x)$, such that the squared $\mathbb{L}^2(P)$ norm of the error is not more 
than this bound.

Let us examine this bound further. Recall the representation of $a_{j_\ell}$ as $V_{j_\ell}^{out} V_{j_\ell}^{in}/V$.  To include the output layer term with the others use $V_{j_0}^{out} V_{j_0}^{in}/V=1$ where the presence of the $V_{j_0}=w_0$ permits an arbitrary factoring of $V$ as $V_{j_0}^{out} V_{j_0}^{in}$.
Accordingly, we have the squared $\mathbb{L}^2(P)$ bound on the error between $f(\tilde a,x)$ and $f(a,x)$ of
$$\frac {1}{M} \left[\sum_{\ell=0}^{L-1} \sum_{j_\ell} \big( V_{j_\ell}^{out} V_{j_\ell}^{in}\big/V)^{1/2} \sigma_{j_\ell} \right]^2.$$
Multiplying by $V^2$ we have the squared error bound between $f(\widetilde W,x)= V f(\tilde a,x)$ and $f(W,x)= V f(a,x)$ of 
$$\frac {V}{M} \left[\sum_{\ell=0}^{L-1} \sum_{j_\ell} \big( V_{j_\ell}^{out} V_{j_\ell}^{in}\big)^{1/2} \sigma_{j_\ell} \right]^2.$$
This is the bound in its refined form from which the other specific bounds \prettyref{eq:bound2} and \prettyref{eq:bound3} are obtained.  

As we have said, we can use that the $\sigma_{j_\ell}$ are not more than $1$, along with the relationship between geometric and arithmetic means, $(V_{j_\ell}^{out} V_{j_\ell}^{in})^{1/2}\, \le (1/2) (V_{j_\ell}^{out} + V_{j_\ell}^{in})$, with equality if and only if $V_{j_\ell}^{out} = V_{j_\ell}^{in}$, to obtain the squared $\mathbb{L}^2(P)$ error bound of
$$\frac {L^2 V \overline V^2}{M}.$$

The bound with $\overline V^{red}$ is similar.  For each $j_\ell$ we identify a linking $j_{\ell+1}^*$, and use $z_{j_{\ell+1}^*}$ in place of the conditional mean of $z_{j_{\ell+1}}$ in obtaining the bound on $\sigma_{j_\ell}^2 (x)$ of $\sum_{j_{\ell+1}} a_{j_{\ell+1}|j_{\ell}} (z_{j_{\ell+1}} - z_{j_{\ell+1}^*})^2$, which is not more than $4\sum_{j_{\ell+1}\ne j_{\ell+1}^*} a_{j_{\ell+1}|j_\ell}$, uniformly in $x$, which provides a bound on $\sigma_{j_\ell}^2$ of $4 V_{j_\ell}^{in,red} /V_{j_\ell}^{in}$.  This yields the same bounds as before but with a factor of $4$ and $V_{j_\ell}^{in,red}$ replacing $V_{j_\ell}^{in}$.

This completes the proof of \prettyref{thm:main}.

\par\vspace{0.1 in}

As we have seen a key quantity arising from the refined form of bounds is the sum of square roots  
\begin{equation}\sum_{j_\ell} \big( V_{j_\ell}^{out} V_{j_\ell}^{in}\big)^{1/2}.\label{eq:geo} \end{equation}
We close this section with a comment about how the square of this quantity compares to the full network variation $V$.
It is seen that it is at least $V$.  Indeed, expansion of the square of this sum includes at least the terms $\sum_{j_\ell}  V_{j_\ell}^{out} V_{j_\ell}^{in}$ which comprise a representation of $V$, as well as other possibly non-negligible terms.  Accordingly, our squared error upper bounds (in the case that we replace the $\sigma_{j_\ell}$ with $1$), in both geometric and arithmetic forms, are at least as large as $L^2 V^2/M$, where we recognize the similarity with the $V^2/M$ bound that holds in the single hidden-layer case. 

\section{Optimization of Interlayer Scalings} \label{sec:scaling}

It may be difficult to assess the size of the sum of geometric means in \prettyref{eq:geo}. Accordingly, the use of the arithmetic mean bound and its resultant subnetwork variation interpretation is important for examination of the size of the bounds. The idea is that our average variation bound is good if constituent subnetworks likewise have good bounds.


As we have seen, there is a multiplicity of representation of a given network function, because it is unchanged if, at any node $j_\ell$,  we multiply the input links $w_{j_{\ell},j_{\ell+1}}$ by a positive $c_{j_\ell}$ and divide the output links $w_{j_{\ell-1},j_\ell}$ by the same $c_{j_\ell}$, for $\ell=1,2,\ldots,L\!-\!1$.  With the introduction of the auxiliary $w_0$ this interlayer scaling invariance holds for $\ell=0$ as well: for any positive $c_{j_0}=c_0$ the network is unchanged if $w_{j_1}$ is multiplied by $c_0$ and $w_0$ is divided by $c_0$. 

To explore the implications of this multiplicity, consider the refined form of the approximation error bound with the $\sigma_{j_\ell}$ replaced by $1$. In this bound we have the geometric means $\sqrt{V_{j_\ell}^{out} V_{j_\ell}^{in}}$ of the input and output variations at every node $j_\ell$ with $\ell < L$.  This geometric mean is invariant to the rescalings of the weights into and out of this node $\sqrt{(V_{j_\ell}^{out}/c_{j_\ell})(c_{j_\ell} V_{j_\ell}^{in})}$.  As a consequence the refined form of the bound is an invariant, unaffected by the multiplicities of weight representation.

Yet, for the arithmetic mean bound, it is fruitful to explore the choices that produce the smallest average variation. For each node, with $\ell\! <\! L$, the arithmetic mean gives a family of bounds $(1/2)(V_{j_\ell}^{out}/c_{j_\ell} + c_{j_\ell} V_{j_\ell}^{in})$ which does depend on the scaling factor.  Each of these produces an approximation bound $L^2 v^2/M$ with $v^2= V \overline V^2$ depending on the scaling. Optimizing the scaling factor to produce the smallest such arithmetic mean bounds for each node yields $c_{j_\ell} = \sqrt{V_{j_\ell}^{out}/V_{j_\ell}^{in}}$, which equalizes the new input and output variations $c_{j_\ell} V_{j_\ell}^{in} = V_{j_\ell}^{out}/c_{j_\ell}$ and we denote the resultant common value $V_{j_\ell}$.  [These equalized variations are unchanged by the adjustments on previous or later layers because the associated scaling factors and their reciprocals cancel in the products which are summed to form these subnetwork variations.]  The equalizing choices make the arithmetic means bound match the geometric means. In this way, we have produced the smallest average variation $\overline V$ among the multiple representation available by such interlayer rescalings.   In brief, the optimized average variation $\overline V$ reproduces the invariant constructed from the geometric means.

The application of these optimal interlayer scalings at all nodes (past the input layer) produces what we call \emph{canonical weights}, which now satisfy the conservation law $V_{j_\ell}^{in} = V_{j_\ell}^{out}=V_{j_\ell}$, giving rise to $V_\ell^{in} = V_\ell^{out}=V_\ell$, for all $\ell \!<\!L$.  For instance, for $\ell=0$, matching $V_0^{out}=w_0$ and $V_0^{in}$, which are required to have product $V$, makes each match $\sqrt V$. In this canonical case, there is no loss to use the simple average variation form of the bound as it matches the seemingly more complicated refined bound. 

If the $ c_{j_{\ell}} $ depend only on the layer $ \ell $, then the choice $ c_{\ell} = \sqrt{V^{out}_{\ell}/V^{in}_{\ell}} $ yields average variation
$$
\frac{1}{L}\sum_{\ell=0}^{L-1}\sqrt{V^{out}_{\ell}V^{in}_{\ell}}.
$$
If the $ c_{j_{\ell}} $ are all constant and equal to $ c $, then the choice $ c = \sqrt{ V^{out}/ V^{in}} $ produces average variation $ \sqrt{\overline V^{out}\overline V^{in}} $, or,
\begin{equation}
\sqrt{\left\|\frac{1}{L}\sum_{\ell=0}^{L-1}W_0W_1 
\cdots W_{\ell}\right\|_1 \left\|\frac{1}{L}\sum_{\ell=0}^{L-1}W_{\ell+1}W_{\ell+2} \cdots W_L\right\|_1}. \label{eq:sqrtvar}
\end{equation}
Thus, the average variation is a product of Ces\`aro averages of the subnetwork variations emanating out of and flowing into the nodes.

In like fashion, using the bound 
on $\sigma_{j_\ell}$ based on the reduced subnetwork variation, the resulting composite variation bound is entirely analogous, now with the geometric means of $V_{j_\ell}^{out}$ and $V_{j_\ell}^{in,red}$.
Optimizing $\overline V^{red}$ among function-preserving 
scalings produces the equalizing case with $V_{j_\ell}^{out} = V_{j_\ell}^{in,red}$.





\section{Discussion} \label{sec:discussion}

In characterizing the sample complexity of multi-layer networks needed for good generalization, early work centered around their pattern-classifying capacity, or roughly speaking, the number of sample points whose random assignments to two classes (dichotomies) can be implemented by the network. \cite[Section VII]{cover1964}, \cite{cover1968, baum1988} give bounds on the capacity that are linear in the total number of weights. Cover \cite{cover1968} also gives an exact expression for the number of ways $ N $ points in $ d $ dimensions can be partitioned into two sets by a single linear threshold, mainly, $ 2\sum_{k=0}^{d-1} \binom{N-1}{k} $. Using this, he also bounds the number of networks with step activation functions, having $ T $ total number of weights, that can implement the $ 2^N $ functions from a pattern set of $ N $ vectors into $ \{-1, +1\} $ by $ (2N)^T $, with logarithm $ T\log(2N) $.

Work such as \cite{barronbarron1988, 
Barron1991-1} expressed the risk of discretized multilayer learning network estimates in terms of their complexity and approximation tradeoff, shown to be optimal in general \cite{Barron1999}, however the approximation tradeoff was only adequately worked out in cases of one and two hidden-layer networks \cite{Barron1993, Barron1994, Cheang1998
,cheang1999
,Barron2008
,Barron2008-2} often with bounds linear in the input dimension, an exception being the recent work of the authors \cite{Klusowski2018} for one hidden-layer networks with bounds logarithmic in the input dimension.

Other recent work has focused on bounding the Rademacher complexity or VC dimensions of various deep network classes and studying how they can be used to bound the generalization error or out-of-bag error (known in statistics as regret or excess risk) \cite{neyshabur2017, neyshabur2015, golowich2017, bartlett2017, arora2018}. Typical results are of the following form: given access to a sample of size $ n $ drawn from the network, the generalization error scales as $ \calC/\sqrt{n} $, where $ \calC $ is some complexity constant that depends on the parameters of the network. 
Typically, $ \calC $ has exponential dependence on $ L $, either \emph{indirectly} through some (possibly weighted) product of norms $ \prod_{\ell=1}^L\|W_{\ell}\| $ of the weight matrices $ W_{\ell} $ across the layers $ \ell = 1, 2, \dots, L $, or \emph{directly} as $ c^L $ for some positive constant $ c > 1 $ (in addition to polynomial factors in $ L $ or logarithmic factors in $ d_1, d_2, \dots, d_L $). Notably, the work of \cite{golowich2017} manages to avoid this direct dependence on $ L $ by controlling various Schatten norms\footnote{The $ p $-Schatten norm of a matrix is the $ \ell^p $ norm of its singular values.} of the weight matrices. 

Perhaps most relevant to our work is \cite[Theorem 3.3]{bartlett2017}, which gives an $\epsilon$-covering entropy bound for depth $ L $ networks (with respect to the empirical measure on $ n $ data points $ \boldsymbol{X} = (X_1, X_2, \dots, X_n) \in \mathbb{R}^{d_{in} \times n} $) of\footnote{In that paper, the network is multi-output and hence the penultimate layer weights $ W_1 $ form a $ d_0 \times d_1 $ matrix instead of a $ d_1 $-dimensional row vector. Nevertheless, for comparison, we treat $ \|W_1\|_{\sigma} $ as the $ \ell^1 $ norm of the row vector $ W_1 $.}
\begin{equation} \label{eq:bartlett}
\frac{\|\boldsymbol{X}\|^2_2\log(\max_{1 \leq \ell\leq L}d_{\ell})}{\epsilon^2}\!\left(\prod_{\ell=1}^{L}\|W_{\ell}\|^2_{\sigma}\!\!\right)\!\!\left(\sum_{\ell=1}^L\!\left(\frac{\|W_{\ell}\|_{2, 1}}{\|W_{\ell}\|_{\sigma}} \right)^{\tfrac{2}{3}}\right)^{\!3}\!\!\!,
\end{equation}
where $ \|W_{\ell}\|_{\sigma} $ and $ \|W_{\ell}\|_{2, 1} = \sum_{j_{\ell-1}}\sqrt{\sum_{j_{\ell}}w^2_{j_{\ell-1},j_{\ell}}} $ are the the spectral\footnote{The spectral norm of $ W_{\ell} $ is the square root of the largest eigenvalue of $ W^{\top}_{\ell}W_{\ell} $.} and $ (2, 1) $ group norms of a matrix, $ \|\boldsymbol{X}\|_2 = \sqrt{\frac{1}{n}\sum_{i=1}^n \|X_i\|^2_2} $. Since $ \|W_{\ell}\|_{2, 1} \geq \|W_{\ell}\|_{\sigma} $, it follows that \prettyref{eq:bartlett} is at least
\begin{equation} \label{eq:bartlett2}
\frac{L^3\|\boldsymbol{X}\|^2_2\log(\max_{1 \leq \ell\leq L}d_{\ell})}{\epsilon^2}\left(\prod_{\ell=1}^{L}\|W_{\ell}\|^2_{\sigma}\right).
\end{equation}
Extensions of these covering bounds to more general matrix norms are given in \cite[Appendix A.5]{bartlett2017}, but they still depend multiplicatively on the depth $ L $ according the product of the individual norms of the weight matrices.
Their proof relies on a matrix covering bound for the affine transformation of each layer via Jones-Barron sparsification and then an inductive argument on the layers to give a covering bound for the whole network.

We see three major advantages of \prettyref{eq:cover1} over \prettyref{eq:bartlett2}. 

\begin{enumerate}
\item First, our bound depends only on $ d_1, d_2, \dots, d_L $ through 
$$ (L-2)\log(\min\{\bar d, 2L^2v^2/\epsilon^2\}) + \log(8e\,d_{L}), $$
with $\bar d \le \max\{d_2,\ldots,d_{L-1}\}$, whereas \prettyref{eq:bartlett2} depends on them more strongly via 
$$ L\log(\max\{d_1, d_2, \dots, d_L\}), $$
most noticeably when $ d_1, d_2, \dots, d_{L-1} $ are larger than $2L^2v^2/\epsilon^2$. In particular, our bound is completely independent of the penultimate layer dimension $ d_1 $. 

\item Second, the factor $ \prod_{\ell=1}^{L}\|W_{\ell}\|^2_{\sigma} $ in \prettyref{eq:bartlett2} results from creating a net for each layer inductively. In contrast, our technique relies on the probabilistic method by sampling $ M $ random paths simultaneously from a distribution that incorporates interactions across successive layers. In the canonical representation of the network, the average variation $\overline V$ is the entry-wise $ \ell^1 $ norm of the Ces\`aro average of successive matrix products $ W_0W_1 
\cdots 
W_{\ell} $ and $ W_{\ell+1}W_{\ell+2}\cdots W_L $. So $ \overline V $ can be thought of as a \emph{norm of a matrix product}, whereas $ \prod_{\ell=1}^{L}\|W_{\ell}\|_{\sigma} $ is a \emph{product of matrix norms}. There may be instances where $ \overline V $ depends only polynomially on $ L $, whereas $ \prod_{\ell=1}^{L}\|W_{\ell}\|_{\sigma} $ may grow exponentially (as is the case when each $ \|W_{\ell}\|_{\sigma} $ is greater than $ 1 $).\footnote{Incidentally, in \cite[Section 4]{bartlett2017}, the authors pose the question of whether there are better choices of norms that would yield smaller complexity constants.} See also \prettyref{eq:mat1} and \prettyref{eq:mat2} in \prettyref{app:appendix} for bounds on the entry-wise $ \ell^1 $ norms of $ W_2W_3 \cdots W_{\ell} $ and $ W_{\ell+1}W_{\ell+2}\cdots W_L $ in terms of $ \|W_2\|_{\sigma}\|W_3\|_{\sigma}\cdots \|W_{\ell}\|_{\sigma} $ and $ \|W_{\ell+1}\|_{\sigma}\|W_{\ell+2}\|\cdots \|W_L\|_{\sigma} $, respectively. These inequalities also show that the average variation is controlled whenever the product of the matrix norms is, but not necessarily the other way around. We will explore this matter further in \prettyref{sec:examples}. 

Finally, let us mention that because of the equivalence of norms, the entry-wise $ \ell^1 $ matrix norm of a product can be bounded by the product of other individual matrix norms. Taken together, these facts imply that the product of the weight matrices is a fundamental quantity, from which other bounds in the literature can be obtained.


\item Finally, they assume that the $\ell^2$ norm of the data matrix $ \|X\|_2 $ scales as some constant $ B $, independent of $ d_{in} $, (to avoid dimension dependence like, say, $ \sqrt{d_{in}} $) necessitating that most of the coordinates of each data vector are of order $1/\sqrt {d_{in}}$. 
Our fixed-width hypercube input domain fits more comfortably within the realm of natural data-generating distributions encountered in high-dimensional data modeling.
\end{enumerate}

We have seen how the complexity constants in our approximation bounds involve the subnetwork variations through their average. In the next section, we give some insights into how this quantity behaves for various choices of weight matrices. 

\section{Example Average Variation Calculations} \label{sec:examples}

In \prettyref{sec:discussion}, we briefly mentioned the advantages of the average variation over other measures of network complexity, defined through products of norms of the individual weight matrices. The average variation is defined through a product of matrices and hence it is not hard to see that there may be cases where the norm of product of matrices is significantly smaller than the product of norms of the individual matrices. Characterization of the growth of a matrix product can be accomplished by either studying conditions on the individual matrices or more global measures on the entire product. To the first point, by \prettyref{eq:productbound0} in \prettyref{app:appendix}, we have 
$$
\|W_1W_2 \cdots W_{\ell}\|_1 \leq \|W_1\|_1\|W_2 \cdots W_{\ell}\|_{1, \infty},
$$
where $ \|A\|_{1, \infty} = \max_{j_1} \sum_{j_2}|a_{j_1, j_2}| $ is the induced norm of a matrix $ A $.
Likewise, by \prettyref{eq:productbound} in \prettyref{app:appendix}, we could give a further bound of 
\begin{equation}
\|W_1\|_1\|W_2\|_{1, \infty} \cdots \|W_{\ell}\|_{1, \infty}, \label{eq:mainproductbound}
\end{equation}
however, we would not want to do so if each factor were larger than numbers near $1$ (or, equivalently, if some of the row sums of the weight matrices were larger than numbers near $1$), as the resulting product could be large. Likewise, by \prettyref{eq:productbound2} in \prettyref{app:appendix}, we can bound $ \|W_{\ell+1}W_{\ell+2}\cdots W_L \|_1 $ by 
\begin{equation} 
\|W_{\ell+1}\|_1\|W_{\ell+2}\|_1\cdots \|W_L\|_1  \label{eq:mainproductbound2},
\end{equation}
although we also may not want to for the same reason. In summary, it is better to convey size through norms of products rather than products of norms.

Inequalities \prettyref{eq:mainproductbound} and \prettyref{eq:mainproductbound2} reveal that the composite variation can be controlled via the maximum row sums or the entry-wise $ \ell^1 $ norms of the weight matrices. On the other hand, these criteria may be too crude and more refined calculations may be necessary. Let us investigate a few of these situations in the sequel. 

The study of the size of products of matrices is facilitated by eigen-decomposition. The role of eigen-decomposition is most apparent in the case of matrices with common eigenvectors, as then the eigenvalues of the product are the product of the eigenvalues. Moreover, the eigenvector with the largest product of eigenvalues controls the size of the resulting product.  For example, if each intermediate layer weight matrix is Toeplitz, with the $ W_{\ell}[j,k] $ depending only on $ j $ and $ k $ through the difference $ j-k $, (as in the case of a class of convolutional networks), then by Szeg\" o's limit theorem \cite[Theorem 5.10]{bottcher1999}, asymptotically, the discrete Fourier basis provides a common eigen-structure and the product of the weight matrices is indeed controlled by the product of the Fourier coefficients as eigenvalues.  Or even more simply, if the weight matrices $ W_{\ell} = Q $ were the same across the intermediate layers, then the shared eigenvector property would be automatic, and the size of the product would be controlled by the size of the largest eigenvalue. Some examples are given in \prettyref{app:appendix}.



Let us now give a concrete example in which the average variation is bounded by a constant, independent of $ L $, but the product of individual matrix norms grows exponentially with $ L $. This example further elucidates the advantages that the average variation has over other measures of network complexity, defined through such products of norms of the weight matrices \cite{neyshabur2017, golowich2017, bartlett2017}.

In this simple example we have in mind, the $ W_{\ell} $ are chosen to be constant across each layer and equal to a square matrix $ Q $ with nonnegative entries. If $ \|Q\| > 1 $, then $ \prod_{\ell=1}^{L}\|W_{\ell}\| = \|Q\|^{L} $ increases exponentially with $ L $, even though the maximum eigenvalue of $ Q $ may be less than or equal to $ 1 $, which is the criterion for stable matrix products and hence stable average variation.

Moreover, since the maximum eigenvalue of $ Q $ is bounded by any sub-multiplicative matrix norm\footnote{A matrix norm $ \|\cdot \| $ is sub-multiplicative if it satisfies $ \|A_1A_2\| \leq \|A_1\|\|A_2\| $ for any matrices $ A_1 $ and $ A_2 $.}, it follows that $ \overline V $ is bounded by a constant, independent of $ L $, whenever $ \prod_{\ell=1}^{L}\|W_{\ell}\| $ is, i.e., when $ \|Q\| \leq 1 $, but not the other way around, as the following example clarifies. Let $ Q = 
\begin{bmatrix} 
t & t(1-t)/s \\
s & 1-t
\end{bmatrix}
$ for $ t \in [0, 1] $ and $ s > 0 $ and also let $ W_0W_1 = \begin{bmatrix} 1 & 1 \end{bmatrix} $. It is easily verified that $ Q $ is a projection matrix and so by \prettyref{ex:projection} in \prettyref{app:appendix}, 
\begin{equation*}
\overline V \approx \sqrt{W_0\|W_1Q\|_1\|Q\|_1} = 1+s+\frac{t(1-t)}{s}.
\end{equation*}
But if $ \| \cdot \| $ is the spectral norm or any Schatten norm (which includes the nuclear and Frobenius norms as special cases), then $$ \|Q\|^2 = ((t-1)^2+s^2)(1+(t/s)^2) > 1, $$ whenever $ s^2 \neq t(1-t) $.

As an aside, note that $ \|Q\|_{1, \infty} = \max\big\{t+ \tfrac{t(1-t)}{s}, s+1-t \big\} > 1 $ whenever $ s \neq t $ and $ \|Q\|_1 > 1 $. Thus, using the product of the individual entry-wise $ \ell^1 $ norms of $ W_{\ell} $ to bound the average variation (as in \prettyref{eq:mainproductbound} and \prettyref{eq:mainproductbound2}) would also lead to undesirable exponential dependence on the depth $ L $.

When more realistically, the $ W_{\ell} $ have distinct eigen-structures across the layers $ \ell $, then there is no shared basis for exact quantification of size via eigenvalues. Nevertheless, products of the maximal eigenvalues can be used as a bound on the variation \cite{daubechies1992}, and aid in the comparison with what others have done, as explained in \prettyref{sec:discussion}.

\subsection{Reduced Variation}

It can be important to use the reduced variation instead of the full variation. As a simple example,
suppose each weight matrix $ W_{\ell} $ is equal to the identity matrix $ I_{d_{in}} $, except for the final row vector $W_1$ which takes a linear combination of the inputs. Then an easy calculation (see \prettyref{ex:projection} in \prettyref{app:appendix}) reveals that $ \overline V^{out} = \frac{1}{L}W_0 + \frac{L-1}{L}W_0\|W_1\|_1 $ and $ \overline V^{in} = \frac{1}{L}\|W_1\|_1 + \frac{L-1}{L}d_{in} $, and hence in accordance with the form \prettyref{eq:sqrtvar}, the average variation is
\begin{equation}
\sqrt{\left(\frac{1}{L}W_0 + \frac{L-1}{L}W_0\|W_1\|_1\right)\left(\frac{1}{L}\|W_1\|_1 + \frac{L-1}{L}d_{in}\right)}. \label{eq:identity}
\end{equation}
In this case, the average variation is of order $ \sqrt{d_{in}} $, which is problematic if $ d_{in} $ is large.

We can overcome this by using the reduced average variation $ \overline V^{red} $, whereby we eliminate the weights of the diagonal cross links, i.e., $ w_{j_{\ell}, j_{\ell+1}^*} $, where $ j_{\ell} =  j_{\ell+1}^* $. Thus, in effect, we may replace $ W_{\ell+1} $ by a hollow matrix $ W^*_{\ell+1} = W_{\ell+1}  - \mbox{diag}(W_{\ell+1} ) $ and instead consider the variation associated with $ W_{\ell+1}  - \mbox{diag}(W_{\ell+1} ) $. We proceed by eliminating the diagonal weights of $ W_{\ell+1} $. Returning to our example, the reduced average variation $\overline V^{red} $ is
$$
\sqrt{\left(\frac{1}{L}W_0 + \frac{L-1}{L}W_0\|W_1\|_1\right)\left(\frac{1}{L}\|W_1\|_1\right)},
$$
which is independent of $ d_{in} $.

Similar conclusions also hold if each $ W_{\ell} = I_{d_{in}} + Q_{\ell} $, $ \ell = 2, 3, \dots, L $, is near the identity matrix, in which the products are controlled by sums of the perturbations from the identity (the perturbation matrices $ Q_{\ell} $ have non-negative entries). One can expand the product of these sums of identity matrix plus small perturbation matrix as follows:
\begin{equation}
W_{\ell+1}W_{\ell+2}\cdots W_L = I_{d_{in}} + \sum_{k=1}^{L-\ell}\sum_{\ell_1, \ell_2, \dots, \ell_k}Q_{\ell_1}Q_{\ell_2}\cdots Q_{\ell_k}, \label{eq:idenprod}
\end{equation}
where the second sum runs over all distinct $ k $-tuples $ (\ell_1, \ell_2, \dots, \ell_k) $ from $ \{\ell+1, \ell+2, \dots, L \} $. The entry-wise $ \ell^1 $ norm of this product is easily seen to be
$$
d_{in} + \sum_{k=1}^{L-\ell}\sum_{\ell_1, \ell_2, \dots, \ell_k}\|Q_{\ell_1}Q_{\ell_2}\cdots Q_{\ell_k}\|_1.
$$
As with \prettyref{eq:identity}, we also see that this would lead to average variation of order $ \sqrt{d_{in}} $. Again, luckily, this strong dimension dependence can be avoid by working with the reduced variation, which has the effect of removing the identity matrix from the decomposition \prettyref{eq:idenprod}.

Suppose $ \sum_{\ell=1}^LQ_\ell $ has entry-wise $ \ell^1 $ norm bounded by $ S $.
In this case, note that by \prettyref{eq:productbound2} in \prettyref{app:appendix}, each term $ \|Q_{\ell_1}Q_{\ell_2}\cdots Q_{\ell_k}\|_1 $ is bounded by $ \|Q_{\ell_1}\|_1\|Q_{\ell_2}\|_1\cdots \|Q_{\ell_k}\|_1 $ and hence,
\begin{align}
& \sum_{k=1}^{L-\ell}\sum_{\ell_1, \ell_2, \dots, \ell_k}\|Q_{\ell_1}Q_{\ell_2}\cdots Q_{\ell_k}\|_1
\nonumber \\ & \leq \sum_{k=1}^{L-\ell}\sum_{\ell_1, \ell_2, \dots, \ell_k}\|Q_{\ell_1}\|_1\|Q_{\ell_2}\|_1\cdots \|Q_{\ell_k}\|_1 \nonumber \\ & 
= \prod_{k=1}^{L-\ell}(1+\|Q_k\|_1) - 1 \leq \exp\left\{ \sum_{k=1}^{L-\ell} \|Q_k\|_1 \right\} \leq e^{S}. \label{eq:outbound}
\end{align}

Now, if $ W^*_{\ell+1} = W_{\ell+1} - \mbox{diag}(W_{\ell+1}) = (Q_{\ell+1}-\mbox{diag}(Q_{\ell+1})) $, then multiplying \prettyref{eq:idenprod} (when $ \ell $ is replaced by $ \ell+1 $) by $ W^*_{\ell+1} $ leads to the expansion
\begin{align}
& W^*_{\ell+1}W_{\ell+2}\cdots W_L = (Q_{\ell+1}-\mbox{diag}(Q_{\ell+1})) + \nonumber \\ & \qquad \sum_{k=1}^{L-\ell-1}\sum_{\ell_1, \ell_2, \dots, \ell_k}(Q_{\ell+1}-\mbox{diag}(Q_{\ell+1}))Q_{\ell_1}Q_{\ell_2}\cdots Q_{\ell_k}, \label{eq:prodrep}
\end{align}
where the second sum runs over all distinct $ k $-tuples $ (\ell_1, \ell_2, \dots, \ell_k) $ from $ \{\ell+2, \ell+3, \dots, L \} $. Using the inequality $ \|(Q_{\ell+1}-\mbox{diag}(Q_{\ell+1}))A\|_1 \leq \|Q_{\ell+1}A\|_1 $ for any matrix $ A $ with nonnegative entries, it is seen that \prettyref{eq:outbound} bounds the $ \ell^1 $ norm of \prettyref{eq:prodrep}. Hence we obtain
$$
\overline V^{in, red} \leq \frac{1}{L}\|W_1W_2\cdots W_L\|_1 + \frac{L-1}{L}e^{S}.
$$
By \prettyref{eq:productbound0} in \prettyref{app:appendix}, $ \|W_1W_2\cdots W_L\|_1 $ is further bounded by $ \|W_1\|_1\|W_2\|_{1, \infty}\cdots \|W_L\|_{1, \infty} $. Furthermore, each $ \|W_{\ell}\|_{1,\infty} $ is equal to $ 1 + \|Q_{\ell}\|_{1,\infty} $, which is also bounded by $ 1 + S $, per our assumptions on $ Q_{\ell} $. This shows that $ \|W_1W_2\cdots W_L\|_1 $ is at most $ \|W_1\|_1\prod_{k=1}^{L-\ell}(1+\|Q_k\|_{1, \infty}) \leq \|W_1\|_1e^{S} $. Hence we obtain
$$
\overline V^{in, red} \leq \frac{\|W_1\|_1}{L}e^{S} + \frac{L-1}{L}e^{S}.
$$
Using the same reasoning for $ W_0W_1\cdots W_{\ell} $, we can bound $ \overline V^{out} $ by $ \frac{1}{L}W_0 + \frac{1}{L}W_0\|W_1\|_1 + \frac{L-2}{L}W_0\|W_1\|_1e^{S} \leq \frac{1}{L}W_0 + \frac{L-1}{L}W_0\|W_1\|_1e^{S}  $.
Plugging in the relevant expressions for $ \overline V^{out} $ and $ \overline V^{in, red} $ into $ \overline V = \sqrt{\overline V^{out}\overline V^{in, red}} $, we obtain a final bound on $ \overline V $ of
\begin{align}
& \sqrt{\frac{W_0}{L} + \frac{L-1}{L}W_0\|W_1\|_1e^{S} } \nonumber \\ & \qquad \times \sqrt{\frac{\|W_1\|_1}{L}e^{S} + \frac{L-1}{L}e^{S} }, \label{eq:avvarbound}
\end{align}
which does not depend explicitly on $ d_{in} $.

Uniform perturbations from the identity can be imposed if the entry-wise $ \ell^1 $ norm of each $ Q_{\ell} $ is at most $ 1/L $. This yields $ S = 1 $. However, the perturbation matrix $ Q_{\ell} $ should have the flexibility to be smaller for inner layers and larger for outer layers. To this end, suppose that $ Q_{\ell} = c_{\ell}\widetilde Q_{\ell} $ for some positive sequence $ c_{\ell} $ that tapers to zero and $ \|\widetilde Q_{\ell}\|_1 \leq 1 $. It follows that $ S $ is at most $ \sum_{\ell=1}^L c_{\ell} $. 

Mild growth of the average variation is permitted if the $c_\ell$ are harmonic weights $1/\ell$. In this case, $ e^{S} \leq L $, so for these specifications, the average variation bound \prettyref{eq:avvarbound} assumes the form
\begin{equation*}
\sqrt{\left(\frac{W_0}{L} + (L-1)W_0\|W_1\|_1\right)\left(\|W_1\|_1 + L-1\right)}.
\end{equation*}

In the next section, we analyze networks with only a single hidden layer. Because a large body of work has been devoted to studying what types of functions are well-approximated by them, we are able to provide covering entropy bounds for a very large class of high-dimensional functions.

\section{Specialization to Two-layer Networks} \label{sec:twolayer}

Let us now consider a simple case of particular interest. When $ L = 2 $, the network functions have the form
\begin{equation} \label{eq:linear}
f(W, x) = \sum_{j_1} w_{j_1} \phi( \sum_{j_2} w_{j_1,j_2} x_{j_2} ).
\end{equation}
These functions are single-hidden layer networks with ramp activation functions.
It is known that if a general function $ f $ admits a Fourier representation $ f(x) = \int_{\mathbb{R}^{d_{in}}}e^{2\pi \mathrm{i}\langle x, \omega\rangle}\calF(f)(\omega)d\omega $ on $ [-1, 1]^{d_{in}} $, then the spectral condition 
$$ C_{f, 2} \triangleq \int_{\mathbb{R}^{d_{in}}}\|\omega\|^2_1|\calF(f)(\omega)|d\omega < +\infty, $$ is sufficient to ensure that it can be well-approximated by a network of the form \prettyref{eq:linear}. More precisely, there exists $ W $ with $ W_0 = C_{f, 2} $, $ \|W_1\|_1 \leq 1 $, and $ \|W_2\|_{1,\infty} \leq 2 $, such that the squared $ \mathbb{L}^2(P) $ distance between $ f(W, x) $ and $ f(x) - f(0) - \langle x, \nabla f(0) \rangle $ is less than $ 16C^2_{f,2}/d_1 $ \cite{KlusowskiBarron2017-2}, \cite[Theorem 3]{Breiman1993}. These weight matrices produce subnetwork variations $ V^{out}_1 = \frac{1}{2}(W_0 + W_0\|W_1\|_1) $ and $ V^{in}_1 = \frac{1}{2}(\|W_1W_2\|_1+\|W_2\|) $ bounded by $ C_{f, 2} $ and $ d_1+1 \leq 2d_1 $, respectively, and variation $ V = W_0\|W_1W_2\|_1 $ bounded by $ 2C_{f, 2} $. Hence, the composite variation $ \overline V\sqrt{V} = \sqrt{\overline V^{out}\overline V^{in} V} $ is bounded by $ 2C_{f, 2}\sqrt{d_1} $.

Furthermore, by \prettyref{thm:main}, for this $ f(W, x) $ there exists another approximant $ f(\widetilde{W},x) $ from a subcollection of log-cardinality at most $ M\log(8e \,d_{in}) $ such that the squared $ \mathbb{L}^2(P) $ distance between $ f(W,x) $ and $ f(\widetilde{W}, x) $ is at most $ 16d_1C^2_{f,2}/M $. By the triangle inequality, if $ f(0) = 0 $ and $ \nabla f(0) = 0 $,  we have
\begin{align*}
\| f - f(\widetilde{W}, \cdot) \| & \leq \| f - f(W, \cdot) \| + \| f(W, \cdot) - f(\widetilde{W}, \cdot) \| \\  & \leq \sqrt{16C^2_{f,2}/d_1} + \sqrt{16d_1C^2_{f,2}/M}.
\end{align*}
Let $ d_1 = \sqrt{M} $ and $ \epsilon = 8C_{f, 2}/M^{1/4} $, so that $ \| f - f(\widetilde{W}, \cdot) \| \leq \epsilon $.
If $ \calF_v $ is the set of all functions $ f $ with $ C_{f, 2} \leq v $, then the previous argument shows that
\begin{equation} \label{eq:cover3}
\calV_{\calF_v}(\epsilon) \leq \frac{8v^4}{\epsilon^4}\log(8e \,d_{in}).
\end{equation}

The collection $ \calF_v $ can also be viewed as the closure of the convex hull class of $ \calF_{2,\, v'} $\,, for some composite variation $ v' $ equal to a multiple of $ v $. There has been a long history in obtaining entropy covering bounds for these function classes \textbf{---} the best available bound is from \cite[Theorem 2.1(1)]{Mendelson2002}, which states that there exists a universal constant $ C > 0 $ such that 
\begin{equation} \label{eq:cover4}
\calV_{\calF_v}(\epsilon) \leq Cd_{in}(v/\epsilon)^{2-4/(d_{in}+2)}.
\end{equation}
When $ d_{in} $ is large relative to $ v^2/\epsilon^2 $, as would be the case in a high-dimensional setting, we see that \prettyref{eq:cover4} is inferior to \prettyref{eq:cover3}.

In the next section, we apply our covering entropy bounds to obtain risk bounds for deep network function classes. As we will see, these risk bounds turn out to be governed by the aforementioned $ (1/2) $ exponent from \prettyref{sec:introduction}.

\section{Implications for statistical learning} \label{sec:statistics}
Statisticians and applied researchers are frequently concerned with predicting a response variable at a new input from a set of data collected from an experiment or observational study. Data are of the form $ \calD_n = \{ (X_i, Y_i) \}_{i=1}^n $, drawn independently from an unknown joint distribution $ P_{X,Y} $ on $ [-1,1]^{d_{in}} \times \mathbb{R} $. The target function is $ f(x) = \mathbb{E}[Y|X=x] $, the mean of the conditional distribution $ P_{Y|X=x} $\,, optimal in mean square for the prediction of future $ Y $ from corresponding input $ X $. In some cases, assumptions are made on the error of the target function $ \varepsilon_i = Y_i - f(X_i) $ (i.e., bounded, Gaussian, sub-Gaussian, or sub-exponential). 

From the data, estimators $ \hat{f}(x) = \hat{f}(\calD_n, x) $ are formed and the loss at a target $ f $ is the $ \mathbb{L}^2(P_X) $ square error $ \|f-\hat{f}\|^2 $ and the risk is the expected squared error $ \mathbb{E}\|f-\hat{f}\|^2 $. For any class of functions $ \mathcal{F} $ on $ [-1,1]^{d_{in}} $, the minimax risk is 
\begin{equation} \label{eq:minimax}
R_n(\mathcal{F}) = \inf_{\hat{f}} \sup_{f\in \mathcal{F}} \mathbb{E}\|f-\hat{f}\|^2,
\end{equation}
where the infimum runs over all estimators $ \hat{f} $ of $ f $ based on the data $ \calD_n $.

For Gaussian errors $ \varepsilon = Y - f(X) $, \cite{Barron1999} exploits the fact that the Kullback-Leibler (K-L) divergence between $ P_{X, Y} $ and $ P_{X, Y'} $, with $ Y' = g(X) + \varepsilon $, is equivalent to the squared $ \mathbb{L}^2(P_X) $ distance between $ f $ and $ g $. This useful observation establishes a direct link between function and density estimation. Their estimation scheme is as follows. First, a joint density $ \hat{p}_n(x, y) $ is estimated, having the form $ \hat{h}(y | x)P(dx) $, where $ \hat{h}(y | x) $ is an estimate of the conditional density of $ Y $ given $ X $. Then the regression estimate $ \hat{f}(x) $ is set to be the minimizer over all $ z $ of the Hellinger distance between $ \hat{h}(\cdot | x) $ and a normal density with mean $ z $ and variance $ \mbox{VAR}(\varepsilon | x) $. It is finally shown that the maximum squared $ \mathbb{L}^2(P_X) $ risk for $ \hat{f} $ is upper bounded by a constant multiple of the Bayes average redundancy for the model. Using the fact that the Bayes mixture density minimizes the average K-L divergence over all choices of densities, the redundancy can be further bounded by a multiple of $ \epsilon^2_n $, where $ \calV_{\calF}(\epsilon_n) \asymp n\epsilon^2_n $. Thus, we can deduce the following results from \cite{Barron1999} and the $\epsilon$-covering entropy bounds \prettyref{eq:cover1} and \prettyref{eq:cover3} by solving $ \calV_{\calF}(\epsilon_n) \asymp n\epsilon^2_n $ for each function class $ \calF $.

For the next set of results in \prettyref{thm:singleupper} and \prettyref{thm:deeprate}, we consider estimation of network functions subjected to Gaussian noise $ N(0, \sigma^2) $. The more general cases of bounded, sub-Gaussian, and sub-exponential errors can be handled by large deviation theory, although we do not present it here.

We first present a risk bound for two layer networks in \prettyref{sec:two-layer} followed by a corresponding result for multi-layer networks in \prettyref{sec:multi-layer}.

\subsection{Two-layer Networks} \label{sec:two-layer}

For the next result, let $ \calG_v $ denote the collection of all functions in $ \calF_v $ with $ \mathbb{L}^{\infty} $ norm at most $ B $. We also define $ \calG_{\infty} = \bigcup_{v>0}\calG_v $.

\begin{theorem} \label{thm:singleupper}
\begin{equation} \label{eq:riskbound2layer}
R_n(\calG_v) \leq C\left(\frac{v^4\log(8e \,d_{in})}{n}\right)^{1/3},
\end{equation}
for some positive constant $ C $ that depends only on $ B $ and $ \sigma^2 $.
\end{theorem}

\begin{remark}
Using the alternative covering entropy bound in \prettyref{eq:cover4} yields
$$
R_n(\calG_v) \leq Cv^{\frac{d_{in}}{d_{in}+1}}\left(\frac{d_{in}}{n}\right)^{1/2 + 1/(2(d_{in}+1))},
$$
for some positive constant $ C $ that depends only on $ B $ and $ \sigma^2 $.
For large $ d_{in} $, this bound is roughly equal to $ v(d_{in}/n)^{1/2} $, and hence we see that \prettyref{eq:riskbound2layer} is superior roughly when $ d_{in} > (nv^2)^{1/6} $.
\end{remark}

\begin{remark}
Compare the rate in \prettyref{thm:singleupper} with the more familiar 
\begin{equation} \label{eq:barronrate}
v\left(\frac{d_{in}\log(en/d_{in})}{n}\right)^{1/2}
\end{equation}
from \cite[Theorem 3]{Barron1994}. There are two main differences \textbf{---} (I) it is only valid for $ n $ moderately large compared to $ d_{in} $ (whereas \prettyref{thm:singleupper} is valid for either regime) and (II) it is for the class of all functions satisfying
$$ C_{f, 1} \triangleq \int_{\mathbb{R}^{d_{in}}}\|\omega\|_1|\calF(f)(\omega)|d\omega \leq v < +\infty,
$$
which is a weaker requirement than $ C_{f, 2} \leq v < +\infty $ since $ C_{f, 1} \leq \sqrt{C_{f, 2}}\;\|\sqrt{|\calF(f)|} \| $ by the Cauchy-Schwarz inequality.

\end{remark}

\begin{remark}
The rate in \prettyref{thm:singleupper} is akin to those obtained in \cite[Theorem 4]{Wainwright2011} or \cite[Theorem 3.2]{rigollet2011} for squared-error prediction in high-dimensional linear regression with $ \ell^1 $ controls on the parameter vectors. However, there is an important difference \textbf{---} the richness of $ \calG_{v} $ is largely determined by the variation through $ v $ and therefore it more flexibly represents a larger class of functions, far beyond the rigidity of linear.
\end{remark}

\begin{remark}
The same rate in \prettyref{eq:riskbound2layer} is available from an adaptive risk bound which holds for estimators that minimize the penalized empirical risk over a finite $\epsilon_n$-cover $ \calU_{\calG_{\infty}}(\epsilon_n) $ of $ \calG_{\infty} $, i.e., when $ \hat{f} $ satisfies (or approximately satisfies via a greedily obtained variant)
\begin{align}
& \frac{1}{n}\sum_{i=1}^n(Y_i - \hat{f}(X_i))^2 + C^{4/3}_{\hat{f}, 2}\psi^{1/3}_n\leq \nonumber \\ & \qquad \inf_{g\in \calU_{\calG_{\infty}}(\epsilon_n)}\left\{ \frac{1}{n}\sum_{i=1}^n(Y_i - g(X_i))^2 + C^{4/3}_{g, 2}\psi^{1/3}_n  \right\},\label{eq:penalized}
\end{align}
where $ \psi_n \asymp (\log(e \, d_{in}))/n $. 

In this case, $ \hat{f} $ has the following adaptive risk bound:
\begin{equation*}
\expect{\|\hat{f}-f\|^2} \leq \inf_{g \in \calG_{\infty}}\left\{ \|g - f\|^2 + C^{4/3}_{g, 2}\psi^{1/3}_n \right\}.
\end{equation*}

If the penalized empirical risk minimization \prettyref{eq:penalized} is performed over the entire space $ \calG_{\infty} $, it can be shown using the covering entropy in \prettyref{eq:cover3} and techniques from \cite{Klusowski2018} that the risk has penalty $ C_{g, 2}\,\psi^{1/4}_n $, and consequently worse rate, with exponent $ (1/4) $ instead of $ (1/3) $. 


\end{remark}

\subsection{Multi-layer Networks} \label{sec:multi-layer}
For multi-layer networks, we let $ \calG_{L, v} $ be the set of all function in $ \calF_{L, v} $ with $ \mathbb{L}^{\infty} $ norm at most $ B $. We also define $ \calG_{L, \infty} = \bigcup_{v>0}\calG_{L, v} $. The next set of results bound the risk for the function class $ \calG_{L, v} $.

\begin{theorem} \label{thm:deeprate}
$$
R_n(\calG_{L, v}) \leq CLv\left(\frac{(L-2)\log(\bar d)+\log(8e \,d_{in})}{n}\right)^{1/2},
$$
where $ C $ is a positive constant that depends only on $ B $ and $ \sigma^2 $.
A bound that is independent of $ \bar d $ is also available, viz.,
$$
R_n(\calG_{L, v}) \leq
CLv\left(\frac{(L-2)\log(v\sqrt{n})+\log(8e \,d_{in})}{n}\right)^{1/2},
$$
where $ C $ is a positive constant that depends only on $ B $ and $ \sigma^2 $.

\end{theorem}


These results are surprising, since they show that the effect of large depth $ L $ and interlayer dimensions $ d_1, d_2, \dots, d_L $ are relatively harmless and benign, even for modest sample sizes. That is, aside for the composite variation $ v $, the minimax risk scales as the previously advertised rate of $ [(L^3\log d)/n]^{1/2} $, where $ d $ is the maximum input dimension of the layers. This may explain why the performance of deep networks does not seem to be hindered by their highly parameterized structure. Also important to notice is that the rate in the exponent $ (1/2) $ does not degrade with the input dimension or other network parameters. What does matter is the average variation, which, as we have argued, can be controlled in many situations.

\subsection{Adaptive Risk Bounds}

All the multi-hidden-layer-network risk bounds we have stated thus far are derived from non-adaptive estimators. Let us briefly make a statement about adaptive risk bounds. It can be shown using techniques from \cite{Klusowski2018} that if $ \hat{f} $ is a penalized estimator with penalty defined through the ``smallest" composite variation $ v(f) $ (or reduced composite variation $ v^{red}(f) $) among all representations of a network $ f $,
\begin{align*}
& \frac{1}{n}\sum_{i=1}^n(Y_i - \hat{f}(X_i))^2 + v(\hat{f})\psi^{1/2}_n \leq \nonumber \\ & \qquad \inf_{g\in \calU_{\calG_{L, \infty}}(\epsilon_n)}\left\{ \frac{1}{n}\sum_{i=1}^n(Y_i - g(X_i))^2 + v(g)\psi^{1/2}_n \right\},
\end{align*}
where $ \psi_n \asymp \frac{(L-2)\log(\bar d)+\log(8e \,d_{in})}{n} $,
then
\begin{equation*}
\expect{\|\hat{f}-f\|^2} \leq \inf_{g \in\calG_{L, \infty}}\left\{ \|g - f\|^2 + v(g)\psi^{1/2}_n \right\}.
\end{equation*}

An important aspect of the above adaptive risk bound is that $ f $ need not belong to $ \calG_{L, \infty} $. The only requirement is that it is well-approximated by certain members of $ \calG_{L, \infty} $.


\subsection{Rademacher Complexity and Generalization Error}

Because of the close connection between covering entropy and generalization error, our covering entropy bounds can also be used to improve the aforementioned generalization bounds in the literature (see \prettyref{sec:discussion}). To be more specific, for a class of functions $ \calF $ and data $ \calD_n = \{(X_i, Y_i) \}_{i=1}^n $, define the \emph{empirical Rademacher complexity} as $ \widehat{\calR}(\calF) = E_{\underline \sigma}\left[{\sup_{f\in\calF} \frac{1}{n}\sum_{i=1}^n \sigma_i f(X_i)}\right] $, where $ \underline \sigma = (\sigma_1, \sigma_2, \dots, \sigma_n) $ is a sequence of iid random variables that assume the values $ \pm 1 $ with equal probability. Then by \cite[Lemma A.5]{bartlett2017}, we can use a refined version of the standard Dudley entropy integral approach to bound the empirical Rademacher complexity via
\begin{align*}
\widehat{\calR}(\calG_{L, v}) & \leq \inf_{\alpha \geq 0}\left( 4\alpha + 12\int_{\alpha}^{\sup_{f\in\calG_{L, v}}\|f\|}\sqrt{\frac{\calV_{\calG_{L, v}}(\epsilon)}{n}}d\epsilon\right).
\end{align*}
Next, using \prettyref{eq:cover1} with metric space $ \mathbb{L}^2(P_n) $, we have that $ \int_{\alpha}^{\sup_{f\in\calG_{L, v}}\|f\|}\sqrt{\calV_{\calG_{L, v}}(\epsilon)}d\epsilon $ is at most
\begin{equation*}
\int_{\alpha}^{B}\sqrt{\frac{L^2v^2[(L-2)\log(\min\{\bar d, \frac{2L^2v^2}{\epsilon^2}\})+\log(8e \,d_{in})]}{\epsilon^2}}d\epsilon.
\end{equation*}
This integral is less than
\begin{equation*}
\int_{\alpha}^{B}\sqrt{\frac{L^2v^2[(L-2)\log(\min\{\bar d, \frac{2L^2v^2}{\alpha^2}\})+\log(8e \,d_{in})]}{\epsilon^2}}d\epsilon,
\end{equation*}
which equals
\begin{equation*}
\log(\tfrac{B}{\alpha})\sqrt{L^2v^2[(L-2)\log(\min\{\bar d, \tfrac{2L^2v^2}{\alpha^2}\})+\log(8e \,d_{in})]}.
\end{equation*}

Finally, choose $ \alpha $ to be proportional to $ 1/\sqrt{n} $. Thus, there exists a constant $ C > 0 $ that depends polylogarithmically on $ B $ such that $ \widehat{\calR}(\calG_{L, v}) $ is at most
\begin{equation*}
CLv(\log n)\sqrt{\frac{(L-2)\!\log(\min\{\bar d, \, 2L^2v^2n\}) + \log(8e \,d_{in})}{n}}.
\end{equation*}
Using standard techniques \cite{mohri2012}, it is possible to convert these bounds into bounds on the generalization error, which are similar in form.



More recently, \cite[Theorem 4.1]{arora2018} used a compression approach to prove that the generalization error is of order\footnote{We ignore some unrelated logarithmic terms and, as before, treat the network as single-output, even though the multi-output case was considered in the paper.}
\begin{equation} \label{eq:aroragen}
\sqrt{\frac{L^2K^2B^2\log(\max_{1 \leq \ell\leq L}d_{\ell})\sum_{\ell=1}^L \frac{1}{\mu^2_{\ell}\mu^2_{\ell\rightarrow}}}{n}},
\end{equation}
where $ B $ is a bound on the $ \mathbb{L}^{\infty} $ norm of the networks, $ \mu_{\ell} $ is a ``layer cushion'', $ \mu_{\ell \rightarrow} $ is an ``interlayer cushion'', and $ K $ is an ``activation contraction'' coefficient. Since they take $ \mu_{\ell\rightarrow} $ to be at most $ 1/\sqrt{d_{\ell}} $, it follows that \prettyref{eq:aroragen} is at least
$$
\sqrt{\frac{L^2K^2B^2\log(\max_{1 \leq \ell\leq L}d_{\ell})\sum_{\ell=1}^L \frac{d_{\ell}}{\mu^2_{\ell}}}{n}}.
$$

Here again we find that the bounds depend more strongly on the interlayer dimensions than ours do. Because it is possible for ramp networks to achieve their variation $ V $ at some input (i.e., saturable networks), we can think of $ B $ here as corresponding to a bound on the variation of the networks.

One question the curious reader may ask is whether it is possible to improve \prettyref{thm:singleupper} or \prettyref{thm:deeprate}. We will now show that the exponent $ (1/3) $ from \prettyref{thm:singleupper} is not improvable beyond $ (1/2) $, in a minimax sense.

\section{Optimality} \label{sec:optimal}

We now show that the minimax risk from \prettyref{thm:singleupper} is nearly optimal, up to a polynomial factor in $ v $. Using the fact that ramp functions are invariant under composition, $ \phi(z) = \phi(\phi(z)) $ for all real $ z $, the constructions used for the following lower bound can easily be adapted to give corresponding statements for depth $ L $ networks.

Below we give an improvement of a minimax lower bound that first appeared in the authors' conference paper \cite[Theorem 1]{Klusowski2017}. For completeness, we reproduce the argument, but tailored to our setting.

\begin{theorem} \label{thm:lower}
Consider the model $ Y = f(X) + \varepsilon $ for $ f \in \calG_v $, i.e., $ C_{f, 2} \leq v $ and the $ \mathbb{L}^{\infty} $ norm of $ f $ is at most $ B $, where $ \varepsilon \sim N(0, \sigma^2) $ and $ X \sim \mbox{Uniform}([-1,1]^{d_{in}}) $.
If $ d_{in} $ is large enough so that 
$$ d_{in} > cA(n/\sigma^2)^{1/A}B^{2/A}\left[ \log(d_{in}/A+1) \right]^{-1/A}-A, $$ where $ A = \sqrt{v/B} $ and $ c $ is a universal positive constant, then
\begin{equation} \label{eq:lowerhighdim}
R_n(\calG_v) \geq C\sigma v^{1/4}\left(\frac{\log ({d_{in}}/\sqrt{v}+1)}{n}\right)^{1/2},
\end{equation}
where $ C $ is a positive constant that depends only on $ B $.
\end{theorem}

Before we prove \prettyref{thm:lower}, we first state a lemma which is contained in the proof of Theorem 1 (pp. 46-47) in \cite{Gao2013}. \\

\begin{lemma} \label{lmm:subsets}
For integers $ N, T $ with $ N \geq 10 $ and $ 1 \leq T \leq N/10 $, define the set
\begin{equation*}
\mathcal{S} = \{ a\in \{0,1\}^N: \|a\|_1 = T \}.
\end{equation*}
There exists a subset $ \calA \subset \calS $ with cardinality at least $ \sqrt{\tbinom{N}{T}} $ such that the Hamming distance between any pairs of $ \calA $ is at least $ T/5 $.
\end{lemma}
Note that the elements of the set $ \calA $ in \prettyref{lmm:subsets} can be interpreted as binary codes of length $ N $, constant Hamming weight $ T $, and minimum Hamming distance $ T/5 $. These are called constant weight codes and the cardinality of the largest such codebook, denoted by $ A(N, T/5, T) $, is also given a combinatorial lower bound in \cite{Sloane1980}. The conclusion of \prettyref{lmm:subsets} is $ A(N, T/5, T) \geq \sqrt{\tbinom{N}{T}} $.

\begin{proof}[Proof of \prettyref{thm:lower}] Define the collection $ \Lambda = \{ \theta \in \mathbb{Z}^{d_{in}} : \|\theta\|_1 \leq A \} $, where $ A \in \mathbb{Z}^{+} $. Then, we have from \cite[Theorem 6]{bump2000} the following expression and lower bound for the number of lattice points in an $ \ell^1 $ ball with radius $ A $:
\begin{align*}
N \triangleq \#\Lambda & = \sum_{k=0}^{\min\{d_{in}, A\}} 2^k \binom{d_{in}}{k}\binom{A}{k} \\ & = \sum_{k=0}^{\min\{d_{in}, A\}} \binom{A}{k}\binom{d_{in}+A-k}{A} \geq \binom{{d_{in}}+A}{A}.
\end{align*}
Consider sinusoidal ridge functions $ \sin(2\pi\langle\theta, x \rangle) $ with $ \theta $ in $ \Lambda $. Note that these functions (for $ \theta \neq 0 $) are orthonormal with respect to the uniform probability measure $ P $ on $ D = [-1, 1]^{d_{in}} $. This fact is easily established using an instance of Euler's formula $ \sin(2\pi\langle\theta, x \rangle) = \frac{1}{2 \mathrm{i}}(\prod_{j=1}^{d_{in}}e^{\mathrm{i}2\pi\theta_j x_j } - \prod_{j=1}^{d_{in}}e^{-\mathrm{i}2\pi\theta_j x_j }) $. 

For an enumeration $ \theta_1, \dots, \theta_M $ of $ \Lambda $, define a subclass of $ \calG_v $ by
\begin{equation*}
\calF_0 = \left\{ f_{a} = \frac{B}{T}\sum_{k=1}^Ma_k\sin(2\pi\langle\theta_k, x \rangle) : a \in \calA \right\},
\end{equation*}
where $ \calA $ is the set in \prettyref{lmm:subsets}. Next, we give a bound on the variation of members from $ \calF_0 $.

\begin{lemma}
For any $ f_{a} $ in $ \mathcal{F}_0 $, we have $ V_{f_{a}, 2} \leq BA^2 $.
\end{lemma}
\begin{proof}
The proof is based on the fact that
$$
\calF(\sin(2\pi\langle\theta,  x \rangle))(\omega) = \frac{\prod_{j=1}^{d_{in}}\delta(\omega_j-\theta_j) - \prod_{j=1}^{d_{in}}\delta(\omega_j+\theta_j)}{2\mathrm{i}},
$$
where $ \delta $ is the Dirac delta function.
Thus, we have
\begin{align*}
V_{f_{a}, 2} & = \int_{\mathbb{R}^{d_{in}}}\|\omega\|^2_1|\calF(f_{a})|d\omega \\ & \leq \frac{B}{T}\sum_{k=1}^N|a_k|\|\theta_k\|^2_1 \leq BA^2. \qedhere
\end{align*}
\end{proof}

Let us furthermore assume that $ A = \sqrt{v/B} $ so that $ \calF_0 \subset \calG_v $.

Any distinct pairs $ f_{a},f_{a'} $ in $ \mathcal{F}_0 $ have $ \mathbb{L}^2(P) $ squared distance at least $ \|f_{a}-f_{a'}\|^2 \geq B^2\|a-a'\|^2_2/T^2 \geq B^2/(5T) $. A separation of $ \epsilon^2 $ determines $ T = (B/(\sqrt{5}\epsilon))^2 $. 

By \prettyref{lmm:subsets}, a lower bound on the cardinality of $ \calA $ is $ \sqrt{\tbinom{N}{T}} $ with logarithm lower bounded by $  (T/2)\log(N/T) $. To obtain a cleaner form that highlights the dependence on $ T $, we assume that $ T \leq \sqrt{N} $, giving $ \log(\# \calA) \geq (T/4)\log N $. Since $ T $ is proportional to $ (B/\epsilon)^2 $, this condition puts a lower bound on $ \epsilon $ of order $ BM^{-1/4} $. Also, $ N \geq (d_{in}/A+1)^A $. Thus, if $ \epsilon > B/({d_{in}}/A+1)^{A/4} $, it follows that a lower bound on the logarithm of the packing number is of order $ \log \calN(\epsilon) = (B/\epsilon)^2A\log({d_{in}}/A+1) $. Thus we have found an $ \epsilon $-packing set of this cardinality. As such, is it a lower bound on the metric entropy of $ \calG_v $.

Next we use the information-theoretic lower bound techniques in \cite{Barron1999} or \cite{Tsybakov2008}. Let $ p_{a}(x,y) = p(x)\psi_{\sigma}(y-f_{a}(x)) $, where $ p $ is the uniform density on $ [-1,1]^{d_{in}} $ and $ \psi_{\sigma} $ is the $ N(0, \sigma^2) $ density. Then
\begin{equation*}
R_n(\calG_v) \geq (\epsilon^2/4)\inf_{\hat{f}}\sup_{f\in\mathcal{F}_0}\mathbb{P}(\|f-\hat{f}\|^2 \geq \epsilon^2),
\end{equation*}
where the estimators $ \hat{f} $ are now restricted to $ \mathcal{F}_0 $. The supremum is at least the uniformly weighted average over $ f \in\mathcal{F}_0 $. Thus a lower bound on the minimax risk is a constant times $ \epsilon^2 $ provided the minimax probability is bounded away from zero, as it is for sufficient size packing sets.
Indeed, by Fano's inequality as in \cite{Barron1999}, this minimax probability is at least
\begin{equation*}
1- \frac{\alpha\log(\# \mathcal{F}_0)+\log2}{\log(\# \mathcal{F}_0)},
\end{equation*}
for $ \alpha $ in $ (0,1) $, or by an inequality of Pinsker, as in Theorem 2.5 in \cite{Tsybakov2008}, it is at least
\begin{equation*}
\frac{\sqrt{\#\mathcal{F}_0}}{1+\sqrt{\#\mathcal{F}_0}}\left(1-2\alpha-\sqrt{\frac{2\alpha}{\log(\#\mathcal{F}_0)}}\;\right),
\end{equation*}
for some $ \alpha $ in $ (0, 1/8) $. These inequalities hold provided we have the following
\begin{equation*}
\frac{1}{\#\mathcal{F}_0}\sum_{a\in\calA}D(p^n_{a}||q) \leq \alpha\log(\#\mathcal{F}_0),
\end{equation*}
bounding the mutual information between $ a $ and the data $ \calD_n = \{(X_i, Y_i)\}_{i=1}^n $, where $ q $ is any fixed joint density for $ \calD_n $. When suitable metric entropy upper bounds on the log-cardinality of covers $ \mathcal{F}_{a'\in\calA'} \triangleq \{ f: \|f-f_{a'}\| < \epsilon' \} $ of $ \mathcal{F}_0 $ are available, one may use $ q $ as a uniform mixture of $ p^n_{a'} $ for $ a' $ in $ \calA' $ as in \cite{Barron1999}, as long as $ \epsilon $ and $ \epsilon' $  are arranged to be of the same order. In the special case that $ \mathcal{F}_0 $ has small radius already of order $ \epsilon $, one has the simplicity of taking $ \calA' $ to be the singleton set consisting of $ a' = 0 $.  In the present case, since each element in $ \mathcal{F}_0 $ has squared norm $ B^2/T = 5\epsilon^2 $ and pairs of elements in $ \mathcal{F}_0 $ have squared separation $ \epsilon^2 $, these functions are near $ f_0 \equiv 0 $ and hence we choose $ q = p^n_0 $.
A standard calculation yields
\begin{equation*}
D(p^n_{a}||p^n_0) \leq \frac{n}{2\sigma^2}\|f_{a}\|^2 \leq \frac{nB^2}{2\sigma^2T} = \frac{5n\epsilon^2}{2\sigma^2}.
\end{equation*}

We choose $ \epsilon_n $ such that this $ (5/(2\sigma^2))n\epsilon^2_n \leq \alpha\log(\# \mathcal{F}_0) $. Thus, in accordance with \cite{Barron1999}, if $ \calN(\epsilon_n) $ is an available lower bound on $ \#\mathcal{F}_0 $, to within a constant factor, a minimax lower bound $ \epsilon^2_n $ on the $ \mathbb{L}^2(P) $ squared error risk is determined by matching
\begin{equation*}
\epsilon^2_n \asymp \frac{\sigma^2\log \calN(\epsilon_n)}{n},
\end{equation*}
Solving this, we find that
\begin{align*}
\epsilon^2_n & \asymp \left(\frac{\sigma^2 B^2A\log({d_{in}}/A+1)}{n}\right)^{1/2} \\ & \asymp \left(\frac{\sigma^2B^{3/2}\sqrt{v}\log(({\sqrt{B}/\sqrt{v})d_{in}}+1)}{n}\right)^{1/2}.
\end{align*}
This quantity is a valid lower bounds on $ R_n(\calG_v) $ to within constant factors, provided $ \calN(\epsilon_n) $ is a valid lower bounds on the $ \epsilon_n $-packing number of $ \calG_v $. Checking that $ \epsilon_n > B/({d_{in}}/A+1)^{A/4} $ yields the condition $ d_{in} > cA(n/\sigma^2)^{1/A}B^{2/A}\left[ \log(d_{in}/A+1) \right]^{-1/A}-A $ for some universal constant $ c > 0 $.
\end{proof}

\begin{remark} The lower bound on $ d_{in} $ is needed to ensure that the lower bounds for the packing numbers take on the form $ T\log N $ instead of $ T\log(N/T) $. We accomplish this by imposing $ T \leq \sqrt{N} $. Alternatively, any upper bound of the form $ N^{\alpha} $, $ \alpha \in (0, 1) $ will work with similar conclusion, adjusting lower bound \prettyref{eq:lowerhighdim} by a factor of $ \sqrt{1-\alpha} $, with corresponding adjustment to the requirement $ d_{in} > cA(n/\sigma^2)^{1/A}B^{2/A}\left[ \log(d_{in}/A+1) \right]^{-1/A}-A $.
\end{remark}

\appendix
\setcounter{equation}{0}
\setcounter{lemma}{0}

\renewcommand{\theequation}{\Alph{section}.\arabic{equation}}
\renewcommand{\thelemma}{\Alph{section}.\arabic{lemma}}
\renewcommand{\theremark}{\Alph{section}.\arabic{remark}}

\section{Appendix} \label{app:appendix}

\subsection{Additional Example Average Variation Calculations}

The reader may be interested in a couple of example settings where the weight matrices are the same across the layers. Each $ W_{\ell} $ is equal to a $ d_{in} \times d_{in} $ matrix $ Q $ for $ \ell = 2, 3, \dots, L $. We either approximate or explicitly calculate the average variation (in accordance with its form in \prettyref{eq:sqrtvar}). First, recall a basic fact from the theory of nonnegative matrices \cite[Theorem 8.3.1]{horn2012} which says that if $ Q $ is a nonnegative matrix with eigenvalues $ \lambda_1, \lambda_2, \dots, \lambda_{d_{in}} $, then the spectral radius $ \rho(Q) = \max\{ |\lambda_1|, |\lambda_2|, \dots, |\lambda_{d_{in}}| \} $ is an eigenvalue of $ Q $, i.e., its largest eigenvalue is real and nonnegative. 



\begin{enumerate}[(A)]



\item {\bf Irreducible matrices:}\label{ex:irreducible} Suppose $ Q $ is an irreducible matrix with nonnegative entries and maximum eigenvalue $ 1 $. 
In the parlance of deep networks, irreducibility here means that for each pair of node indices $ j $ and $ k $, there is a path $ (j_1, j_2, \dots, j_{\ell-1}, j_{\ell}) $ emanating from $ k = j_{\ell} $ and flowing into $ j = j_1 $ such that the product of their corresponding weights $ w_{j_1,j_2}w_{j_2,j_3} \cdots w_{j_{\ell-1},j_{\ell}} $ is strictly positive. 

By \cite[Theorem 8.4.4]{horn2012}, $ Q $ (resp. $ Q^{\top}) $ has a unique eigenvector $ u $ (resp. $ v $) with strictly positive components, corresponding to the largest eigenvalue $ 1 $. The Ces\`aro average of products of $ Q $ is a semi-convergent matrix and by \cite[Theorem 8.6.1]{horn2012}, we have the limit
\begin{equation*}
\lim_{\ell\rightarrow +\infty} \frac{1}{L}\sum_{\ell=1}^LQ^{\ell} = \frac{uv^{\top}}{\langle u, v \rangle},
\end{equation*}
where the convergence rate (with respect to any matrix norm) is $ O(1/L) $.

Hence the entry-wise $ \ell^1 $ norm of Ces\`aro averages
\begin{align*}
\overline V^{out} & = \left\|\frac{1}{L}\sum_{\ell=0}^{L-1}W_0W_1 \cdots W_{\ell}\right\|_1 \nonumber \\ & = \frac{1}{L}W_0+\left\|\frac{1}{L} \sum_{\ell = 0}^{L-2} W_0W_1Q^{\ell}\right\|_1,
\end{align*}
and
\begin{align*}
\overline V^{in} & = \left\|\frac{1}{L}\sum_{\ell=0}^{L-1}W_{\ell+1}W_{\ell+2} \cdots W_L\right\|_1 \nonumber \\ & = \frac{1}{L}\|W_1Q^{L-1}\|_1 + \left\|\frac{1}{L} \sum_{\ell = 1}^{L-1} Q^{\ell}\right\|_1
\end{align*}

converge to 
\begin{equation*}
\frac{\langle u, W_0W_1 \rangle \|v\|_1}{\langle u, v \rangle}
\end{equation*} 
and
\begin{equation*} 
\frac{\|u\|_1\|v\|_1}{\langle u, v \rangle},
\end{equation*} 
respectively,
as $ L $ approaches infinity. 
Consequently, for large $ L $, $ \overline V $ is approximately equal to $$ \frac{\sqrt{W_0\langle u, W_1 \rangle\|u\|_1 } \|v\|_1}{\langle u, v \rangle}.$$




\item {\bf Projection matrices:}\label{ex:projection}
Suppose $ Q $ is a projection matrix, i.e., $ Q^2 = Q $ and has nonnegative entries. Then $ \overline V^{out} = \frac{1}{L}W_0 + \frac{1}{L}W_0\|W_1\|_1 + \frac{L-2}{L}W_0\|W_1Q\|_1 $ and $ \overline V^{in} = \frac{1}{L}\|W_1\|_1 + \frac{L-1}{L}\|Q\|_1 $, which means that $ \overline V $ is equal to
\begin{align*}
& \sqrt{\frac{1}{L}W_0 + \frac{1}{L}W_0\|W_1\|_1 + \frac{L-2}{L}W_0\|W_1Q\|_1} \, \times \\ & \qquad
\sqrt{\frac{1}{L}\|W_1Q\|_1 + \frac{L-1}{L}\|Q\|_1} \\ & \approx
\sqrt{W_0\|W_1Q\|_1\|Q\|_1}
\end{align*}



\end{enumerate}

\subsection{Supplementary Proofs}
\begin{lemma}
Suppose $ a $ is a $ d_0 $ dimensional row vector and $ A_k $ are $ d_{k-1} \times d_k $ matrices for $ k = 1, 2, \dots, m $. The following inequalities hold:
\begin{align}
\|aA_1A_2\cdots A_m\|_1 & \leq \|a\|_1\|A_1A_2\cdots A_m\|_{1, \infty} \label{eq:productbound0}  \\
& \leq \|a\|_1\|A_1\|_{1, \infty}\|A_2\|_{1, \infty} \cdots \|A_m\|_{1, \infty}, \label{eq:productbound}
\end{align}
\begin{equation}
\|A_1A_2\cdots A_m\|_1 \leq \sqrt{d_m}\|A_1\|_{\sigma}\|A_2\|_{\sigma}\cdots \|A_m\|_{\sigma}, \label{eq:mat1}
\end{equation}
\begin{equation}
\|A_1A_2\cdots A_m \|_1
 \leq \|A_1\|_1\|A_2\|_1\cdots \|A_m\|_1, \label{eq:productbound2}
\end{equation}
and
\begin{equation}
\|A_1A_2\cdots A_m\|_1 \leq d_0\sqrt{d_m}\|A_1\|_{\sigma} \|A_2\|_{\sigma}\cdots \|A_m\|_{\sigma}. \label{eq:mat2}
\end{equation}
\end{lemma}
\begin{proof}
Suppose $ A $ is an $ d_0 \times d_m $ matrix with entries $ a_{j_1, j_2} = A [j_1, j_2] $. Since $ \|aA\|_1 \leq \|a\|_1\|A\|_{1, \infty} $, where $ \|A\|_{1, \infty} = \max_{j_1} \sum_{j_2}|a_{j_1,j_2}| $, we have 
\begin{equation} \label{eq:mat3}
\|aA_1A_2\cdots A_m\|_1 \leq \|a\|_1\|A_1A_2\cdots A_m\|_{1, \infty},
\end{equation}
which shows the first inequality in \prettyref{eq:productbound0}.  Since $\| \cdot \|_{1, \infty} $ is an induced matrix norm, it is sub-multiplicative \cite[Example 5.6.5]{horn2012}, and hence we have the further bound involving the product of the norms \prettyref{eq:productbound}.

By \cite[Exercise 5.6.P5 and 5.6.P23]{horn2012}, $ \|A\|_{1, \infty} \leq \sqrt{d_m}\|A\|_{\sigma} $ and hence 
\begin{equation} \label{eq:mat5}
 \|A_1A_2\cdots A_m\|_{1, \infty} \leq \sqrt{d_m} \|A_1A_2\cdots A_m\|_{\sigma}, 
\end{equation} which is further upper bounded by the products of the individual norms,
\begin{equation} \label{eq:mat4}
\sqrt{d_m} \|A_1\|_{\sigma}\|A_2\|_{\sigma}\cdots \|A_m\|_{\sigma},
\end{equation}
again, since $ \| \cdot \|_{\sigma} $ is also an induced matrix norm and hence is sub-multiplicative.
Taken together, inequalities \prettyref{eq:mat5} and \prettyref{eq:mat4} imply \prettyref{eq:mat1}. 

Inequality \prettyref{eq:productbound2} follows from the fact that the entry-wise $ \ell^1 $ matrix norm is sub-multiplicative \cite[Example, p. 341]{horn2012}.

Finally, \prettyref{eq:mat2} can be shown by noting that $ \|A_1A_2\cdots A_m\|_1 $ is bounded by
$$
d_0\|A_1A_2\cdots A_m\|_{1, \infty},
$$
and then using inequalities \prettyref{eq:mat5} and \prettyref{eq:mat4} to bound $ \|A_1A_2\cdots A_2\|_{1, \infty} $ by $ \sqrt{d_m} \|A_1\|_{\sigma}\|A_2\|_{\sigma}\cdots \|A_m\|_{\sigma} $.
\end{proof}

\bibliographystyle{plain}
\bibliography{ref}

\end{document}